\pdfoutput=1

\documentclass[11pt]{article}

\usepackage{acl}

\usepackage{times}
\usepackage{latexsym}

\usepackage[T1]{fontenc}

\usepackage[utf8]{inputenc}

\usepackage{microtype}

\usepackage{inconsolata}

\usepackage{graphicx}

\usepackage{custom}
\input{macros}

\usepackage{pgfplots}
\pgfplotsset{compat=1.18}
\usepgfplotslibrary{groupplots}

\pgfplotsset{
  frontier/local/.style  ={thick, mark=*,      mark size=2.2pt},
  frontier/hybrid/.style ={thick, mark=square*,mark size=2.2pt},
  frontier/nope/.style   ={only marks, mark=triangle*, mark size=2.6pt},
  frontier/rope/.style   ={only marks, mark=diamond*,  mark size=2.6pt},
  frontier/sin/.style    ={only marks, mark=x,         mark size=2.8pt},
}

\title{Characterizing the Expressivity of Local Attention in Transformers}

\author{
Jiaoda Li%
~\;~\;~Ryan Cotterell\\
\texttt{\{\href{mailto:jiaoda.li@inf.ethz.ch}{jiaoda.li}, \href{mailto:ryan.cotterell@inf.ethz.ch}{ryan.cotterell}\}@inf.ethz.ch}\\
    {%
\setlength{\fboxsep}{2.5pt}%
\setlength{\fboxrule}{2.5pt}%
    \includegraphics[width=.15\linewidth]{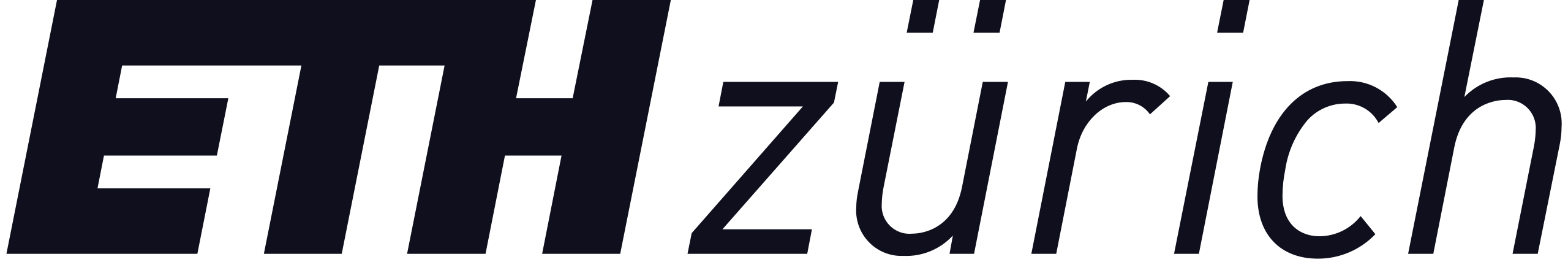}
}}

\begin{document}
\maketitle
\begin{abstract}
The transformer is the most popular neural architecture for language modeling. 
The cornerstone of the transformer is its global attention mechanism, which lets the model aggregate information from all preceding tokens before generating the next token.
One common variant of attention is called local attention, which restricts each token to aggregating information from a bounded window of predecessors, reducing the quadratic cost of global attention to linear. 
Although this restriction is usually motivated by efficiency, it has also been found to improve model quality, a phenomenon that has so far lacked a satisfactory explanation. 
We provide a formal account of this phenomenon in terms of recognizer expressivity.
It has been shown that fixed-precision transformers with global attention correspond to a fragment of linear temporal logic containing a single past operator. We additionally prove that adding local attention introduces a second temporal operator, strictly enlarging the class of recognizable regular languages.
Moreover, global and local attention are expressively complementary: neither subsumes the other, and combining them yields the richest fragment. Experiments on formal language recognition and natural language modeling corroborate the theory, showing that hybrid global--local transformers outperform their global-only counterparts.
\end{abstract}

\section{Introduction}
The transformer \citep{vaswani2017} is the dominant neural architecture in modern language modeling \citep{gpt1,gpt2,gpt3,gpt4}.
At its core lies the attention mechanism \citep{schmidhuber1992learning,Bahdanau2015}, which allows the model to aggregate information from all previously generated tokens before generating the next one.
\emph{Local attention} \citep{luong-etal-2015-effective,DBLP:journals/corr/abs-1904-10509,DBLP:journals/corr/abs-2004-05150,NEURIPS2020_c8512d14} is a popular attention variant that restricts the model to aggregating over a fixed-size window of preceding tokens.

Local attention is typically motivated by efficiency.
Standard \emph{global attention} aggregates over \emph{all} previous positions, incurring quadratic cost in sequence length; local attention reduces this to linear.
At first glance, restricting the range of tokens that a transformer can attend to appears to be a strict weakening---a trade-off of expressivity for speed.
Surprisingly, however, local attention has been repeatedly observed to yield empirical improvements, both in machine translation \citep{luong-etal-2015-effective} and in language modeling \citep{DBLP:journals/corr/abs-1904-10509,DBLP:journals/corr/abs-2004-05150,NEURIPS2020_c8512d14}---even when computational budgets are held constant.

We resolve this puzzle by analyzing which formal languages different attention patterns can recognize, a standard lens in the formal study of neural models \citep{deletang2023neural,butoi2025training}. This lens also has practical relevance: modern language models \citep{gpt4,deepseek2025,olmo2025} are routinely used to perform computation via prompting \citep{wei2022chain,li2024chain,merrill2024the}.
In this vein, \citet{li2025} show that fixed-precision transformer language models with global attention correspond exactly to the fragment of linear temporal logic with a single past operator.
Expanding their work, we prove that adding local attention introduces a new temporal operator, enabling the recognition of a strictly larger class of regular languages. These include the locally testable languages, a family closely connected to $n$-gram models and extensively studied in formal language theory
\citep[Ch. 2]{mcnaughton1971counter}.
Crucially, local attention is complementary rather than a drop-in replacement for global attention: with purely local attention, the computation depends only on a fixed-length suffix of the input, collapsing expressivity to the smaller class of definite languages \citep[\S 5]{Kleene1956}.

Empirically, we test these predictions on formal language recognition with length generalization.
We compare transformers with (i) global attention, (ii) local attention, and (iii) a hybrid that mixes global and local attention.
The results align with the theory: hybrid models recognize strictly more languages than either extreme and generalize more reliably to longer sequences. Varying local attention's window size further shows that a window of size one is consistently the strongest choice in the local-attention family, both for purely local models and for hybrid models. Finally, we evaluate models equipped with two popular positional encodings---sinusoidal ($\sincos$) and rotary ($\rotary$)---and find that neither compensates for the absence of local attention.
We further show that these findings extend to natural language: on WikiText-2, hybrid attention with a window of size one consistently achieves the lowest perplexity across positional encodings.

\section{Linear Temporal Logic}
\label{sec:hammer}
In this section, we introduce linear temporal logic, the main tool we use to characterize the expressive power of transformers under different types of attention. 
We examine the fragments relevant to our analysis, prove the separations between them, and connect them to classical characterizations from algebraic formal language theory.

\subsection{Preliminaries}
We now introduce \defn{linear temporal logic} \citep[LTL;][]{Kamp1968-KAMTLA-2,pnueli1977}, which has become a popular tool in the analysis of transformers \citep{yang2024masked,yang2026simulating,li2025}.

\paragraph{Tokens, Strings and Languages.}
An \defn{alphabet} is a finite, non-empty set of \defn{tokens}.
Fix an alphabet $\alphabet$. 
A \defn{string} over $\alphabet$ is a finite sequence of tokens drawn from $\alphabet$.
The \defn{length} of a string $\str = \sym_1 \cdots \sym_N$, denoted $|\str| = N$, is the number of tokens it contains.
We write $\kleene{\alphabet}$ for the set of all strings over $\alphabet$, and call any subset
$\lang\subseteq \kleene{\alphabet}$ a \defn{language}.

\paragraph{Linear Temporal Logic.}
The syntax and semantics of linear temporal logic ($\gtl{\past, \prev, \since, \until}$) as well as the definition of operator depth are recalled in \cref{app:ltl-definitions}.
We also give a definition of what it means for a formula in LTL to define a language; see \cref{app:ltl-semantics}. 
Given a formula $\tlf$, we write $\langfun(\tlf)$ for the set of all strings satisfying $\tlf$.

\paragraph{Fragments of $\gtl{\past, \prev, \since, \until}$.}
For a set of temporal operators $\operators \subseteq \{\past, \prev, \since, \until\}$, we write $\gtl{\operators}$ for the corresponding fragment in which only operators from $\operators$ are allowed. 
\begin{definition}
\label{def:definable}
Let $\operators \subseteq \{\past, \prev, \since, \until\}$ be a set of temporal operators.
A language $\lang$ is \defn{definable} by $\gtl{\operators}$ if
there exists a formula $\tlf$ in $\gtl{\operators}$ such that
$\langfun(\tlf) = \lang$.
\end{definition}

\citet{10.1145/567446.567462} show that every language definable by $\gtl{\past, \prev, \since, \until}$ is also definable by both $\gtl{\since}$ and $\gtl{\until}$; in the paper, we write $\ltl$ for full linear temporal logic.
The fragments $\ptl$ and $\ytl$ are strictly less expressive, i.e., there exist languages definable by $\ltl$ that are not definable by either $\ptl$ or $\ytl$.

\subsection{Characterizing $\ytl$}\label{sec:characterizing-ylt}
We first characterize the expressive power of $\ytl$ by relating it to a classical class of regular languages. Because formulas in $\ytl$ can only look back a bounded number of steps, we show they naturally correspond to languages determined by a bounded suffix---the definite languages.

For a string $\str$ of length $\length$ and an integer $m \le \length$, an \defn{$m$-factor} of $\str$ is any contiguous substring of $\str$ of length $m$. The \defn{$m$-prefix} of $\str$ is its initial substring of length $m$, and the \defn{$m$-suffix} of $\str$ is its final substring of length $m$. We write $\factor_m(\str)$ for the set of all $m$-factors of $\str$, and $\prefix_m(\str)$ and $\suffix_m(\str)$ for the $m$-prefix and $m$-suffix of $\str$, respectively. If $m > |\str|$, we set $\prefix_m(\str)=\suffix_m(\str)=\str$.
Formally, a language $\lang \subseteq \kleene{\alphabet}$ is \defn{$m$-definite}, for some $m \ge 0$, if $\lang$ is a Boolean combination of languages of the form $\left\{\str \in \kleene{\alphabet} \mid \suffix_m(\str)=\stru\right\}$.
A language is \defn{definite} if it is $m$-definite for some $m \ge 0$.

We next relate $\ytl$ to definite languages, for which we also provide algebraic and automata-theoretic characterizations.
The relevant background, including definitions of the syntactic semigroups and deterministic finite automata (DFAs), are recalled in \cref{app:background}.

\begin{restatable}{reTheorem}{ytldefinite}
  \label{thm:ytl-definite}
  Let $\lang \subseteq \kleene{\alphabet}$ be a regular language, let $\semigroup$ be its syntactic semigroup, and let $\automaton$ be the minimal DFA recognizing $\lang$. The following are equivalent:
  \begin{enumerate}[leftmargin=1.75em,labelwidth=1.2em,labelsep=0.5em]
    \item \label{item:ytl} $\lang$ is definable in $\ytl$;
    \item \label{item:definite} $\lang$ is definite;
    \item \label{item:semigroup} every idempotent $\idem\in\semigroup$ is a right zero, i.e., for all $s\in\semigroup$, $s\idem=\idem$;
    \item \label{item:dfa_non_permutation} every transformation induced by a non-empty string is non-permutational.
  \end{enumerate}
\end{restatable}

\begin{proof}
See \cref{app:proof-ytl-definite}.
\end{proof}

\subsection{$\ptl$ and $\ytl$ are Incomparable}

\citet{li2025} give a rich characterization of $\ptl$.
Building on this and \Cref{sec:characterizing-ylt}, we now relate $\ptl$ and $\ytl$.
However, we first discuss what it means to compare two sets of operators for linear temporal logic.
Fix $\operators_1, \operators_2 \subseteq \{\past, \prev, \since, \until\}$.
We say $\gtl{\operators_1}$ is more expressive than $\gtl{\operators_2}$, written in symbols $\gtl{\operators_2} \subseteq \gtl{\operators_1}$, if every language definable in $\gtl{\operators_2}$ is also definable in $\gtl{\operators_1}$.
We say $\gtl{\operators_1}$ is \emph{strictly} more expressive than $\gtl{\operators_2}$, written  in symbols as $\gtl{\operators_2} \subset \gtl{\operators_1}$, if $\gtl{\operators_2} \subseteq \gtl{\operators_1}$ and there exists a language definable in $\gtl{\operators_1}$ but not in $\gtl{\operators_2}$.
We call $\gtl{\operators_1}$ and $\gtl{\operators_2}$ 
\defn{incomparable} if neither is more expressive than the other.

\begin{proposition}
\label{prop:ptl_not_ytl}
There exists a language definable in $\ptl$ that is not definable in $\ytl$.
\end{proposition}

\begin{proof}
Consider the language $\syma\kleene{\alphabet}$, i.e., the set of strings whose first token is $\syma$. It is definable in $\ptl$ by $\past(\atom_\syma \land \lnot \past \true)$, i.e., $\langfun(\past(\atom_\syma \land \lnot \past \true)) = \syma\kleene{\alphabet}$.
However, it is not definite, and therefore not definable in $\ytl$ by \cref{thm:ytl-definite}.
\end{proof}

\begin{proposition}
\label{prop:ytl_not_ptl}
There exists a language definable in $\ytl$ that is not definable in $\ptl$.
\end{proposition}

\begin{proof}
Consider the language $\kleene{\alphabet}\syma$, i.e., the set of strings whose last token is $\syma$. It is definable in $\ytl$ by $\prev \atom_\syma$, i.e., $\langfun(\prev \atom_\syma) = \kleene{\alphabet}\syma$.
However, it is not definable in $\ptl$: \citet{li2025} show that $\ptl$ defines exactly the class of left-deterministic polynomials, and $\kleene{\alphabet}\syma$ is not a left-deterministic polynomial. 
See \cref{app:background} for an exposition of the relevant technical definitions.
\end{proof}

\begin{corollary}
\label{cor:ptl_ytl_incomparable}
The fragments $\ptl$ and $\ytl$ are incomparable.
\end{corollary}

\begin{proof}
This is a direct implication of both \cref{prop:ptl_not_ytl,prop:ytl_not_ptl}.
\end{proof}

\subsection{Characterizing $\yptl$}

We give another implication of \cref{cor:ptl_ytl_incomparable} below.

\begin{corollary}
\label{cor:yptl_strictly_more_expressive}
$\ytl, \ptl \subset \yptl$.
\end{corollary}

\begin{proof}
Note that $\yptl$ trivially contains both $\past$ and $\prev$; it subsumes both fragments.
However, because $\ptl$ and $\ytl$ are incomparable by \cref{cor:ptl_ytl_incomparable}, there exists a language definable in $\yptl$ that is not definable in $\ptl$ as well as a language definable in $\yptl$ that is not definable in $\ytl$.
This renders both subsumptions strict, as we sought to show.
\end{proof}

To further characterize the expressive power of $\yptl$, we relate it to standard algebraic and automata-theoretic notions. The relevant definitions, including locally $\gR$-trivial semigroups and configurations of DFAs, are recalled in \cref{app:background}.
\begin{restatable}{reTheorem}{characterization}
  \label{thm:pptl}
  Let $\lang \subseteq \kleene{\alphabet}$ be a regular language, $\semigroup$ its syntactic semigroup, and $\automaton$ the minimal DFA recognizing $\lang$. The following are equivalent:
  \begin{enumerate}[leftmargin=2em,labelwidth=1.2em,labelsep=0.5em]
    \item \label{item:yptl} $\lang$ is definable in $\yptl$;
    \item \label{item:monoid} $\semigroup$ is locally $\gR$-trivial;
    \item \label{item:dfa} $\automaton$ contains no configuration of the form shown in \cref{fig:forbidden}.
  \end{enumerate}
\end{restatable}

\begin{proof}
See \cref{app:proof-pptl}.
\end{proof}

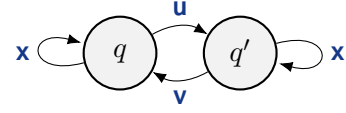
\begin{figure}
  \centering
  \begin{tikzpicture}
      \node[state] (q0) [] { $\stateq$ };
      \node[state] (q1) [right = of q0] { $\stateq'$ };
      \draw[transition]
        (q0) edge[auto, bend left] node{$\stru$} (q1)
        (q1) edge[auto, bend left] node{$\strv$} (q0)
        (q0) edge[auto, loop left] node{$\strx$} (q0)
        (q1) edge[auto, loop right] node{$\strx$} (q1);
  \end{tikzpicture}
  \caption{Forbidden configuration in the minimal DFAs of $\yptl$-definable languages.}
  \label{fig:forbidden}
\end{figure}

Next, we give two examples that respectively illustrate a language definable in $\yptl$ and a language beyond its expressive power.

\subsubsection{The Locally Testable Languages}

The addition of $\prev$ enables the definition of classes of languages that are linguistically relevant yet lie beyond the expressive power of $\ptl$ alone. 
As a representative example, we consider the locally testable languages.

A language $\lang \subseteq \kleene{\alphabet}$ is \defn{locally $m$-testable}, for some $m>0$, if it is a Boolean combination of languages of the following three forms:
\begin{subequations}
\begin{align}
  &\{\str \in \kleene{\alphabet} \mid \stru \in \factor_m(\str)\},
  &&\stru \in \alphabet^m, \label{eq:lt-factor}\\
  &\{\str \in \kleene{\alphabet} \mid \prefix_{m-1}(\str)=\stru\},
  &&\stru \in \alphabet^{m-1}, \label{eq:lt-prefix}\\
  &\{\str \in \kleene{\alphabet} \mid \suffix_{m-1}(\str)=\stru\},
  &&\stru \in \alphabet^{m-1}. \label{eq:lt-suffix}
\end{align}
\end{subequations}
A language is \defn{locally testable} if it is locally $m$-testable for some $m>0$. 
\citet{li2025} show that $\ptl$ cannot define many simple locally testable languages, e.g., $\kleene{\alphabet}\syma$ and $\kleene{\alphabet}\syma\symb\kleene{\alphabet}$, because they are not right-deterministic.
In contrast, as we show in the subsequent theorem, all locally testable languages are definable in $\yptl$, yielding a witness to the separation in expressive power between $\ptl$ and $\yptl$.

\begin{restatable}{reTheorem}{localtest}
  \label{thm:localtest}
  Any locally testable language is definable in $\yptl$.
\end{restatable}

\begin{proof}
See \cref{app:proof-localtest}.
\end{proof}

\subsubsection{A Bounded Dyck Language}

Despite this gain in expressivity, $\yptl$ still defines a limited class of languages.
For instance, it does not define all star-free languages, which are exactly the languages definable in $\ltl$.
One illustrative family is the bounded Dyck languages, consisting of well-balanced parentheses with a fixed upper bound on nesting depth.\footnote{The \defn{nesting depth} of a position is the number of unmatched opening parentheses at that point in the string, and the \defn{maximum nesting depth} of a string is the largest such number attained anywhere in the string.}
The bounded Dyck language of depth $k$ consists of all well-balanced strings whose maximum nesting depth is at most $k$.\footnote{Bounded Dyck languages are often used as simple proxies for hierarchical structure in syntax \citep{hewitt-etal-2020-rnns}.}
For example, over $\{\syma,\symb\}$, where $\syma$ is an opening parenthesis and $\symb$ is a closing parenthesis, the languages
\begin{equation}
  \kleene{(\syma\symb)}
  \qquad\text{and}\qquad
  \kleene{(\syma\kleene{(\syma\symb)}\symb)}
\end{equation}
correspond to maximum nesting depth $1$ and $2$.

The minimal DFAs for these languages are shown in \cref{fig:dfa}. The depth-$1$ language $\kleene{(\syma\symb)}$ is recognizable by $\yptl$, because its minimal DFA avoids the forbidden configuration in \cref{fig:forbidden}. In contrast, the depth-$2$ language $\kleene{(\syma\kleene{(\syma\symb)}\symb)}$ is not $\yptl$-definable: its minimal DFA contains the configuration from \cref{fig:forbidden}, witnessed by $\stateq=\stateq_0$, $\stateq'=\stateq_1$, $\stru=\syma$, $\strv=\symb$, and $\strx=\syma\symb$.

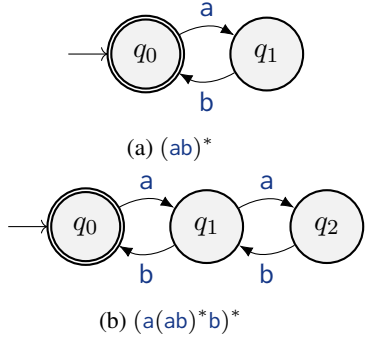
\begin{figure}
  \begin{subfigure}[c]{\columnwidth}
    \centering
    \begin{tikzpicture}
        \node[state, initial, accepting] (q0) [] { $\stateq_0$ }; 
        \node[state] (q1) [right = of q0] { $\stateq_1$ }; 
        \draw[transition] (q0) edge[auto, bend left] node{$\syma$} (q1) 
        (q1) edge[auto, bend left] node{$\symb$} (q0);
    \end{tikzpicture}  
    \caption{$\kleene{(\syma\symb)}$}
  \end{subfigure}
  \begin{subfigure}[c]{\columnwidth}
    \centering
    \begin{tikzpicture}
        \node[state, initial, accepting] (q0) [] { $\stateq_0$ }; 
        \node[state] (q1) [right = of q0] { $\stateq_1$ }; 
        \node[state] (q2) [right = of q1] { $\stateq_2$ }; 
        \draw[transition] (q0) edge[auto, bend left] node{$\syma$} (q1) 
        (q1) edge[auto, bend left] node{$\syma$} (q2) 
        (q1) edge[auto, bend left] node{$\symb$} (q0)
        (q2) edge[auto, bend left] node{$\symb$} (q1);
    \end{tikzpicture}  
    \caption{$\kleene{(\syma\kleene{(\syma\symb)}\symb)}$}
  \end{subfigure}
  \caption{Minimal DFAs. Nodes represent states and arrows represent transitions. The initial state is indicated by an incoming arrow with no source node, and accepting states are shown with double circles.}
  \label{fig:dfa}
\end{figure}

\subsection{Characterizing $\yptl[k]$}
We now introduce two additional temporal operators derived from $\prev$. 
The bounded operator $\prev[k]$ asks whether $\tlf$ holds at some position among the previous $k$ steps, while the unbounded operator $\prevstar$ asks whether $\tlf$ holds at \emph{any} earlier position. Formally, these two operators are defined as follows
\begin{subequations}
\label{eq:derived_operators}
\begin{align}
\label{eq:prevk_def}
\prev[k]\tlf &\defeq \bigvee_{i=1}^{k} \prev^{\,i}\tlf = \bigvee_{i=1}^{k} \underbrace{\prev\prev\cdots\prev}_{i \text{ times}}\tlf & \\
\label{eq:prevstar_def}
\prevstar\tlf &\defeq \bigvee_{i=1}^{\infty} \prev^{\,i}\tlf = \bigvee_{i=1}^{\infty} \underbrace{\prev\prev\cdots\prev}_{i \text{ times}}\tlf &
\end{align}
\end{subequations}
Note that $\prev$ is the special case $\prev[1]$, and $\prevstar$ coincides with $\past$.
This last fact implies that $\gtl{\past, \prevstar} = \ptl$, i.e., the fragments are not only comparable, they are equivalent.
However, for any $k \geq 1$, the fragment $\ytl[k]$, as is the case with $\ytl$, is incomparable with $\ptl$.
We make this claim precise in the following theorem.

\begin{restatable}{reTheorem}{ptlytlkincomparable}
\label{thm:ptl_ytlk_incomparable}
For any $k \geq 1$, the fragments $\ptl$ and $\ytl[k]$ are incomparable.
\end{restatable}

\begin{proof}
See \cref{app:proof-ptl-ytlk-incomparable}.
\end{proof}

Thus, we arrive at the following corollary.
\begin{restatable}{reCorollary}{yptlkstrictly}
\label{cor:yptlk_strictly_more_expressive}
The fragment $\yptl[k]$ is strictly more expressive than each of the fragments $\ptl$ and $\ytl[k]$.
\end{restatable}

\begin{proof}
See \cref{app:proof-yptlk-strictly}.
\end{proof}

We next analyze expressivity for various $k \geq 1$; we find $\prev[k]$ is maximally expressive at $k=1$.
\begin{restatable}{reProposition}{ytlksubsetytl}
  \label{prop:ytlk_subset_ytl}
The fragment $\ytl[k]$ is strictly less expressive than $\ytl$ for $k > 1$.
\end{restatable}

\begin{proof}
See \cref{app:proof-ytlk-subset-ytl}.
\end{proof}

We now show adding $\past$ preserves this strictness.

\begin{restatable}{reProposition}{yptlksubsetyptl}
  \label{prop:yptlk_subset_yptl}
The fragment $\yptl[k]$ is strictly less expressive than $\yptl$ for $k>1$.
\end{restatable}

\begin{proof}
See \cref{app:proof-yptlk-subset-yptl}.
\end{proof}
The gap in expressivity between $\yptl[k]$ and $\yptl$ can, however, be closed in the presence of suitable numerical predicates.
Numerical predicates depend only on positions, and not on the tokens occurring at those positions \citep[II.2]{straubing1994finite}. A particularly important family is modular predicates $\mod[m][r]$, with $m>0$ and $r\geq 0$:
\begin{equation}
  \str,n \models \mod[m][r]
  \quad\text{iff}\quad
  n \equiv r \pmod m .
\end{equation}

\begin{restatable}{reProposition}{modularproposition}
  \label{prop:modular}
For every $k>1$, if modular predicates $\mod[m][r]$ are available for some $m \ge k$, then $\ytl[k]$ has the same expressive power as $\ytl$, and $\yptl[k]$ has the same expressive power as $\yptl$.
\end{restatable}

\begin{proof}
See \cref{app:proof-modular}.
\end{proof}

\section{The Transformer}
\label{sec:nail}
In this section, we formally define the transformer architecture we consider.
Note that we analyze the transformer as a language \emph{recognizer}.

\subsection{Preliminaries}
All arithmetic in this section is carried out over a set of representable numbers $\fpnset$; we defer the precise formalization to \cref{sec:fixed_precision}. For now, the reader may assume that all relevant operations (addition, multiplication, $\exp$, etc.) are well-defined over $\fpnset$.

\begin{definition}[Recognition]
\label{def:recognizer}
Fix an alphabet $\alphabet$. 
A \defn{language recognizer} is a function $\cls\colon\kleene{\alphabet}\to \{0,1\}$. A language $\lang \subseteq \kleene{\alphabet}$ is \defn{recognized} by $\cls$ if $\cls(\str)=1$ for all $\str\in\lang$ and $\cls(\str)=0$ for all $\str\notin\lang$.
\end{definition}

We can also define string recognition over an $\eos$-extended alphabet $\eosalphabet \defeq \alphabet\cup\{\eos\}$, where $\eos\notin\alphabet$ is an \defn{\underline{e}nd-\underline{o}f-\underline{s}equence} token.

\begin{definition}[$\eos$-Recognition]
\label{def:eos_recognition}
Fix an alphabet $\alphabet$.
Let $\cls\colon\kleene{\eosalphabet}\to\{0,1\}$ be a recognizer over the $\eos$-extended alphabet. We say that $\cls$ \defn{$\eos$-recognizes} a language $\lang \subseteq \kleene{\alphabet}$ if $\cls$ recognizes the language $\lang\eos \defeq \{\str\eos \mid \str \in \lang\} \subseteq \kleene{\eosalphabet}$.
\end{definition}

In other words, \Cref{def:eos_recognition} says that, given an input $\str = \sym_1\cdots\sym_\length \in \kleene{\alphabet}$, the recognizer processes the $\eos$-padded string $\eosstr \defeq \sym_1\cdots\sym_\length\eos$.
This convention mirrors the semantics of LTL (see \cref{app:ltl-semantics}): a formula $\tlf$ is evaluated at position $\length+1$, just beyond the last token of $\str$. In the transformer, position $\length+1$ corresponds to $\eos$.

\subsection{Attention}

\paragraph{The Basic Definition.}
Attention can be viewed as a matrix-to-matrix function.
Let $\transformer\in\fpnset^{\dimension \times (\length+1)}$ be a matrix.
We map $\transformer$ to another matrix as follows.
We first define the linear transformations
\begin{subequations}
\label{eq:qkv}
\begin{align}
    \query(\transformer) &\defeq \W[Q]\transformer, \\
    \key(\transformer) &\defeq \W[K]\transformer, \\
    \val(\transformer) &\defeq \W[V]\transformer,
\end{align}
\end{subequations}
where $\W[Q]\in\fpnset^{\dimension[K] \times \dimension}$, $\W[K]\in\fpnset^{\dimension[K] \times \dimension}$, and $\W[V]\in\fpnset^{\dimension[V] \times \dimension}$ are parameter matrices.
These are often called the \defn{\underline{q}uery}, \defn{\underline{k}eys} and \defn{\underline{v}alues}, respectively. 
For integers $n,m\in\{1,\ldots,\length+1\}$, we then define the \defn{\underline{s}core} as follows
\begin{equation}
\label{eq:scores}
  \score_{n,m}(\transformer) \defeq \frac{\query[:][n](\transformer)\bigcdot \key[:][m](\transformer)}{\sqrt{\dimension[K]}} \in \fpnset.
\end{equation}
Next, we define the \defn{\underline{$\boldsymbol{\alpha}$}ttention weights} as
\begin{equation}
\label{eq:softmax_SA}
\aalpha_{n,m}(\transformer) \defeq
\frac{\exp(\score_{n,m}(\transformer)) \mask_{n,m}}
{\sum_{i=1}^{\length+1} \exp(\score_{n,i}(\transformer)) \mask_{n,i}},
\end{equation}
where $\mask \in \{0,1\}^{(\length+1) \times (\length+1)}$ is a binary matrix called the \defn{\underline{m}ask}.
Note that $\aalpha_{n,m}\defeq 0$ for all $m$ if $\mask_{n,i}=0$ for all $i$, i.e., if all preceding positions are masked.
Finally, we define the resulting matrix
\begin{equation}
\label{eq:masked_attention}
\rmO_{:,n}(\transformer) \defeq \sum_{m=1}^{\length+1} \aalpha_{n,m}(\transformer) \val[:][m](\transformer),
\end{equation}
which is the \defn{\underline{o}utput} of the function.
In summary, we define the \defn{\underline{a}ttention mechanism} as follows
\begin{align}
\attention \colon \fpnset^{\dimension \times (\length+1)} &\to \fpnset^{\dimension \times (\length+1)} \notag \\
\transformer &\mapsto \rmO(\transformer).
\end{align}

\paragraph{Masking Patterns.}
We now give an exposition of several different choices of masks $\mask$. 
First, we define the \defn{global mask} 
\begin{equation}
\label{eq:mask_global}
\globalmask[n][m] \defeq
\begin{cases}
1 & \text{if } m < n,\\
0 & \text{otherwise.}
\end{cases}
\end{equation}
which corresponds to global attention in the attention mechanism. 
Next, we define the $k$-\defn{local mask}:
\begin{equation}
\label{eq:mask_local}
\!\!\localmask[n][m] \defeq
\begin{cases}
1 & \text{if } \max(1,n-k) \le m < n,\\
0 & \text{otherwise,}
\end{cases}
\end{equation}
which corresponds to \defn{$k$-local attention}.

\paragraph{Multi-Head Attention.}
It is also commonplace to combine attention 
heads. 
A \defn{multi-head attention} with $\headnumber$ heads $\rmO^{1},\ldots,\rmO^{\headnumber}$ is defined below
\begin{equation}
\label{eq:multihead_combine}
\attention(\transformer) \defeq \sum_{h=1}^{\headnumber} \W^{h}\rmO^{h}(\transformer),
\end{equation}
where the heads are combined by a linear projection, and each $\W^{h}\in\fpnset^{\dimension \times \dimension[V]}$ is a parameter.
Different heads may use different masks: for instance, some heads may use $\globalmask$ while others use $\localmask$, giving rise to \defn{hybrid global--local attention}.

\subsection{Transformer Architecture}
We now define a transformer recognizer.
\begin{definition}[Transformer]
\label{def:transformer}
Fix an alphabet $\alphabet$.
A $(\dimension, \layernumber)$-\defn{transformer} is a 4-tuple composed of
\begin{itemize}[itemsep=4pt,leftmargin=1.5em,labelwidth=1em,labelsep=0.5em]
\item a token encoder $\ve\colon\eosalphabet\to\fpnset^{\dimension}$,
\item a family of attention mechanisms $\{\attention[\ell]\}_{\ell=1}^{\layernumber}$,
\item a family of non-linear functions $\{\ffn[\ell]\}_{\ell=1}^{\layernumber}$ where each $\ffn[\ell]\colon\fpnset^{\dimension} \to \fpnset^{\dimension}$, and
\item a classification function $\classification\colon\fpnset^{\dimension}\to\{0,1\}$.
\end{itemize}
\end{definition}
A transformer $\eos$-recognizes a language as follows.
Given $\eosstr = \sym_1\cdots\sym_\length \eos$ where $\sym_1, \ldots, \sym_{\length} \in \alphabet$, define the base case
\begin{equation}
  \transformer[0][:][n][\eosstr] \defeq \ve(\sym_n), \quad \forall n\in\{1,\ldots,\length+1\}.
\end{equation}
Then, for each layer $\ell=1,\ldots,\layernumber$, compute
\begin{subequations}
\begin{align}
\label{eq:layer_update}
\halftransformer[\ell][][][\eosstr] &\defeq \LN\left(\transformer[\ell-1][][][\eosstr] + \attention[\ell](\transformer[\ell-1][][][\eosstr])\right), \\
\transformer[\ell][][][\eosstr] &\defeq \LN\left(\halftransformer[\ell][][][\eosstr] + \ffn[\ell](\halftransformer[\ell][][][\eosstr])\right),
\end{align}
\end{subequations}
where $\ffn[\ell]$ is applied column-wise, and $\LN$ is layer normalization \citep{ba2016layer}.
The transformer's output is then $\cls(\str) \defeq \classification(\transformer[\layernumber][:][\length+1][\eosstr])$.

\subsection{Fixed Precision}
\label{sec:fixed_precision}
In the exposition above, we assumed that transformers operate at \defn{fixed precision}, i.e., there exists a finite set $\fpnset$ of numbers---analogous to the floating-point formats used in practice \citep{ieee754}---such that all parameters and intermediate values lie in $\fpnset$. Every primitive operation, including $\exp$ and the division in softmax, is performed in $\fpnset$ and rounded back into $\fpnset$. We note that fixed-precision arithmetic differs from exact arithmetic in subtle ways: for instance, addition is not generally associative over $\fpnset$, so the order of summation in softmax can affect the result. Nevertheless, the fixed-precision assumption is standard in the analysis of transformers and is part of what makes the connection to finite automata and temporal logic possible.

\section{Characterizing Transformers}
\label{sec:char}
We now show that global, local, and hybrid attention correspond exactly to the fragments introduced in \cref{sec:hammer}.
We first restate an existing result, relating transformers with global attention to $\ptl$.

\begin{theorem}[\citealp{li2025}]
\label{thm:ptl}
 A language $\lang$ is $\eos$-recognizable by a $(\dimension, \layernumber)$-transformer $\bigl(\ve, \{\attention[\ell]\}_{\ell=1}^{\layernumber}, \{\ffn[\ell]\}_{\ell=1}^{\layernumber}, \classification\bigr)$ in which every attention head uses the global mask $\globalmask$ if and only if it is definable in $\ptl$.
\end{theorem}

We now consider $k$-local attention.

\begin{theorem}
\label{thm:ytlk}
A language $\lang$ is $\eos$-recognizable by a $(\dimension, \layernumber)$-transformer $\bigl(\ve, \{\attention[\ell]\}_{\ell=1}^{\layernumber}, \{\ffn[\ell]\}_{\ell=1}^{\layernumber}, \classification\bigr)$ in which every attention head uses the $k$-local mask $\localmask$ if and only if it is definable in $\ytl[k]$.
\end{theorem}

\begin{proof}[Proof sketch]
The argument follows the same structure as \citet{li2025} for \cref{thm:ptl}.
Because each query attends to at most $k$ predecessors, all summations in the masked softmax and value aggregation reduce to bounded counting, which is expressible in $\ytl[k]$ using bounded nesting of $\prev[k]$. Conversely, by induction on $\ytl[k]$ formulas, Boolean connectives are implemented positionwise, and $\prev[k]\tlf$ is realized by a $k$-local attention head that aggregates information from the last $k$ positions satisfying $\tlf$. The same refinement step as in \citet[Lemma~B.12]{li2025} can be applied.
\end{proof}

Finally, we consider hybrid transformers that mix global and $k$-local heads.

\begin{theorem}
\label{thm:yptlk}
A language $\lang$ is $\eos$-recognizable by a $(\dimension, \layernumber)$-transformer $\bigl(\ve, \{\attention[\ell]\}_{\ell=1}^{\layernumber}, \{\ffn[\ell]\}_{\ell=1}^{\layernumber}, \classification\bigr)$ in which some heads use $\globalmask$ and the remaining heads use $\localmask$ if and only if it is definable in $\yptl[k]$.
\end{theorem}

\begin{proof}
The theorem is a direct corollary of \cref{thm:ptl,thm:ytlk}.
\end{proof}

\subsection{Outlook}

We now translate the logical characterizations into concrete consequences for transformers under different attention patterns.
An immediate consequence of \cref{thm:ptl,thm:ytlk,thm:yptlk} together with \cref{cor:yptlk_strictly_more_expressive} is that transformers with hybrid global-local attention are strictly more expressive than either global-only or $k$-local-only transformers.
Thus, mixing local and global attention is not only computationally attractive \citep{luong-etal-2015-effective,DBLP:journals/corr/abs-1904-10509,DBLP:journals/corr/abs-2004-05150,NEURIPS2020_c8512d14}, but can also yield a strict expressive gain.
Moreover, $1$-local attention is the strongest choice within the local-attention family: by \cref{prop:ytlk_subset_ytl,prop:yptlk_subset_yptl,thm:ytlk,thm:yptlk}, for every $k > 1$, transformers with $1$-local attention are strictly more expressive than those with $k$-local attention, and the same holds when global heads are added.
This suggests that 1-local attention is a particularly promising choice in practice. We note, however, two important caveats.

\begin{figure*}[t]
  \centering
  \includegraphics[width=\textwidth]{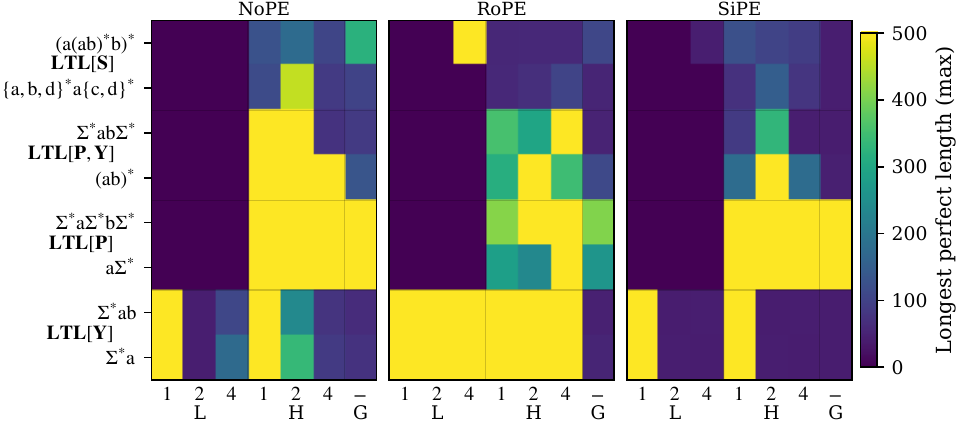}
  \caption{Heatmap of longest perfect lengths (maximum over runs) across formal languages, attention patterns, and positional encodings. Columns denote local (L), hybrid (H), and global (G) attention patterns; for L and H, the labels above columns indicate the local window size $k$.
  }
  \label{fig:heatmap_len_max}
\end{figure*}

\paragraph{Depth Overhead.}
Although $1$-local attention maximizes expressivity, replacing a single $k$-local head by $1$-local heads may incur a depth overhead. At the logical level, expressing $\prev[k]\tlf$ via \cref{eq:prevk_def} requires a disjunction of formulas of the form $\prev^i \tlf$ for $1 \le i \le k$. Thus, using only $\prev$ increases the operator depth (\cref{app:operator-depth}) from $\od(\tlf)+1$ to as much as $\od(\tlf)+k$. In other words, eliminating $\prev[k]$ in favor of $\prev$ preserves expressivity, but may increase the required depth by an overhead of up to $k-1$. 
Specifically, in the proof of \cref{thm:ytlk}, each temporal operator requires at least one layer to simulate, so this translation may require up to $k$ stacked layers in place of a single layer. 
Thus, under a fixed depth, transformers with $1$-local attention need not be strictly more expressive than transformers with $k$-local attention. In practice, modern language models are deep \citep{gpt4}, making this limitation less of a practical detriment and rendering architectures with $1$-local heads a plausible way to increase expressivity.

\paragraph{Positional Encodings.}
The strict separation above holds in the absence of positional encodings. With sufficiently rich positional information, however, the gap can disappear. Intuitively, if the model can always identify the immediate predecessor within a $k$-window, then $k$-local attention can simulate $1$-local attention.
Positional encodings are often viewed as numerical predicates added to the logic \citep{yang2024masked,li2025}. By \cref{prop:modular}, suitable modular predicates suffice to erase the gap. Therefore, if positional encodings that simulate such predicates are available, then a $k$-local head can recover the immediate predecessor by selecting, at each position, the unique position in the previous $k$ steps.
Interestingly, sinusoidal encodings (\sincos; \citealt{vaswani2017}) and rotary encodings (\rotary; \citealt{SU2024127063}) can be related to modular predicates \citep{yang2025kneedeep}. However, realizing modular predicates in this way requires a rational variant \citep{pmlr-v202-chiang23a}, which is not the form typically used in practice.

\section{Empirically Verifying the Theory}
\label{sec:formal_experiments}
In this section, we test the expressivity predictions of our theory on formal language recognition tasks.

\subsection{Languages}

We evaluate a suite of formal languages grouped by definability in temporal-logic fragments, selecting two representative languages from each class:
\begin{itemize}
  \item \textbf{$\ytl$-definable:}
  $\kleene{\alphabet}\syma$ (strings ending with $\syma$) and
  $\kleene{\alphabet}\syma\symb$ (strings ending with $\syma\symb$).

  \item \textbf{$\ptl$-definable:}
  $\syma\kleene{\alphabet}$ (strings starting with $\syma$) and
  $\kleene{\alphabet}\syma\kleene{\alphabet}\symb\kleene{\alphabet}$
  (strings containing $\syma\symb$ as a not-necessarily-contiguous subsequence).

  \item \textbf{$\yptl$-definable but not in $\ptl$ or $\ytl$:}
  $\kleene{(\syma\symb)}$ (strings of repeated $\syma\symb$) and
  $\kleene{\alphabet}\syma\symb\kleene{\alphabet}$
  (strings containing $\syma\symb$ as a contiguous substring).

  \item \textbf{$\ltl$-definable but not in $\yptl$:}
  $\kleene{\{\syma,\symb,\symd\}}\syma\kleene{\{\symc,\symd\}}$
  (a right-deterministic polynomial) and
  $\kleene{(\syma\kleene{(\syma\symb)}\symb)}$
  (bounded Dyck with maximum nesting depth $2$).
\end{itemize}

\subsection{Experimental Setup}

We consider three attention settings, corresponding to the masking patterns defined in \cref{sec:nail}: \emph{local-only} (all heads use $\localmask$), \emph{hybrid} (half the heads use $\localmask$, the rest use $\globalmask$), and \emph{global-only} (all heads use $\globalmask$).
For local-only and hybrid models, we vary the window size $k \in \{1,2,4\}$. We also evaluate transformers with two positional encodings: sinusoidal encodings (\sincos; \citealt{vaswani2017}) and rotary encodings (\rotary; \citealt{SU2024127063}).
The models are trained on strings of length up to $40$ and evaluated on lengths $41$--$500$. Each experiment is run with $5$ random seeds and $3$ learning rates; additional details are given in \cref{app:setup_formal}. In addition to accuracy, we report the \emph{longest perfect length}, defined as the largest length up to which the model achieves $100\%$ accuracy.

\subsection{Results}

We summarize performance using a heatmap of the longest perfect length in \cref{fig:heatmap_len_max}. Full numerical results, including accuracies and longest perfect lengths, are provided in \cref{app:results}.

\subsubsection{Results on $\nope$}
We begin with a discussion of the $\nope$ setting, whose results are shown on the left panel of \cref{fig:heatmap_len_max}.

\paragraph{Local-only models match $\ytl$.}
As predicted, local-only transformers succeed precisely on the $\ytl$-definable languages: they achieve perfect generalization up to length $500$ on the two $\ytl$ tasks in \cref{fig:heatmap_len_max}, while failing to generalize on the remaining languages.

\paragraph{Global-only models match $\ptl$.}
Conversely, global-only models succeed only on the $\ptl$-definable languages and do not exhibit perfect length generalization beyond this class.

\paragraph{Hybrid models match $\yptl$.}
Hybrid models, especially with $k=1$, succeed on all languages definable in $\yptl$, which subsumes both $\ptl$ and $\ytl$. As visualized in \cref{fig:heatmap_len_max}, they generalize perfectly up to length $500$ on the bottom six languages.

\paragraph{$1$-local attention.}
Among the tested window sizes, $k=1$ is consistently the strongest performer. 
Both $k=2$ and $k=4$ fail even on definite languages, and increasing the window can hurt: for example, hybrid models with $k=2$ learn $\kleene{\alphabet}\syma\symb\kleene{\alphabet}$, whereas those with $k=4$ do not.

\begin{figure*}[t]
  \centering
  \includegraphics[width=\textwidth]{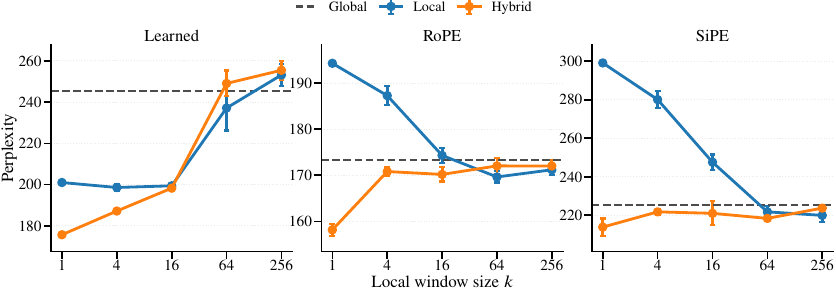}
  \caption{Perplexity on WikiText-2 for local, hybrid, and global attention patterns under different positional encodings. Curves show mean perplexity across runs for local and hybrid models as a function of the local window size $k$, and error bars indicate standard deviations. The dashed horizontal line shows the corresponding global baseline.
  }
  \label{fig:lm-perplexity}
\end{figure*}

\paragraph{Beyond $\yptl$.}
As expected, none of the models perfectly generalizes on the two languages in $\ltl \setminus \yptl$, namely the right-deterministic polynomial and bounded Dyck. Nevertheless, hybrid attention yields the strongest partial generalization on these tasks.

\subsubsection{Positional Encodings}

We now turn to the settings with positional encodings (middle and right panels of \cref{fig:heatmap_len_max}).

\paragraph{Global-only models.}
One might hope that positional encodings would allow global attention to approximate locality. In our experiments, however, neither \sincos\ nor \rotary\ compensates for the absence of local attention: Positional encodings do not erase the qualitative separation between global-only and hybrid models; on the contrary, they often reduce performance on the tested languages.

\paragraph{Local-only and hybrid models.}
With \rotary, the gap between $1$-local and larger-window local attention becomes smaller. In particular, local-only models with $k=2$ or $k=4$ can now recognize $\kleene{\alphabet}\syma$ and $\kleene{\alphabet}\syma\symb$. In hybrid models, \rotary\ also enables $k=4$ to recognize $\kleene{\alphabet}\syma\symb\kleene{\alphabet}$. However, these gains come with tradeoffs, as \rotary\ degrades performance in several other settings.
By comparison, \sincos\ does not narrow the gap and instead tends to reduce performance across the board.

\paragraph{Summary.}
Overall, on one hand, the results under $\nope$ align exactly with our theory. On the other hand, the role of positional encodings is less clear-cut and merits further investigation.

\section{An Experiment on Natural Language}

We also experiment with whether our theory has something to say about natural language.
Thus, we run a lightweight next-token language modeling experiment on WikiText-2.

\subsection{Experimental Setup}

We train a GPT-2-small-style decoder-only Transformer from scratch, using the standard GPT-2 small architecture (12 layers, hidden size 768, 12 attention heads, and feed-forward size 3072). We keep the architecture and optimization setup fixed across all runs, varying only the attention pattern. Concretely, we compare global-only, local-only, and hybrid attention, where in the hybrid setting half of the heads are replaced by local heads. To remain consistent with the formal setup in \cref{sec:nail}, we use strict causal masking throughout. We consider local window sizes $k \in \{1,4,16,64,256\}$ and three positional encoding schemes: learned absolute positional embeddings, \rotary, and \sincos. For optimization, we use AdamW \citep{loshchilov2018decoupled} with learning rate $2.5 \times 10^{-4}$, weight decay $0.01$, 2000 warmup steps, cosine learning-rate decay, per-device batch size 2, gradient accumulation of 32 steps, and early stopping based on validation loss.

\subsection{Results}

The results show a clear and consistent pattern. Across all positional encoding choices, the hybrid model is always the strongest, and within the hybrid family, $k=1$ is always the best setting. Hybrid attention with $k=1$ reduces perplexity by $69.7$, $15.2$, $11.5$ relative to the global-only baseline under various positional encodings. Increasing the local window generally degrades performance, and with sufficiently large windows the hybrid model can match or underperform the global-only baseline. 

We also find that local-only attention can outperform global-only attention, which is surprising but compatible with our theory, since local-only and global-only attention are incomparable in expressive power. One possible explanation is that WikiText-2 relies more heavily on short-range patterns than on the longer-range dependencies. The effect of window size, however, depends on the positional encoding. Under learned positional embeddings, smaller local windows perform better. Under \rotary\ and \sincos, local-only models instead tend to prefer larger windows. This is consistent with the theory: although $1$-local attention is maximally expressive in the unbounded setting, its advantage may disappear under a fixed depth bound, especially when suitable positional encodings narrow the gap between different local window sizes.

Overall, our results support the main claim of the paper: local attention is complementary to global attention, and the hybrid setting combines the strengths of both. In particular, pairing global attention with $1$-local attention yields the largest additional benefit.

\section{Conclusion}
We have given a formal account of why local attention helps in transformer language models. By connecting attention patterns to fragments of linear temporal logic, we showed that local and global attention are complementary rather than interchangeable, and that adding local attention to a globally attentive transformer strictly increases expressive power. In particular, global attention alone corresponds to the logic $\ptl$, purely $k$-local attention corresponds to $\ytl[k]$, and hybrid global--local attention corresponds to the richer logic $\yptl[k]$. The largest gain arises from $1$-local attention, which induces the $\prev$ operator and captures, among other classes, the locally testable languages---a classical family beyond the reach of global-only models. Our formal-language experiments empirically confirm these separations, and a preliminary language-modeling experiment on WikiText-2 shows that hybrid attention with $k{=}1$ consistently achieves the best performance.

\section*{Limitations}
Due to limited computational resources, we are unable to run large-scale language modeling experiments. Instead, we use a lightweight setup that is intended as a sanity check while remaining manageable in terms of model size, dataset, and training budget. Accordingly, our natural-language results should be interpreted with caution. It is possible that the qualitative picture changes with larger models, longer training, or broader datasets.

\section*{Ethical Considerations}
We do not foresee any ethical issues.

\bibliography{custom}

\appendix
\onecolumn

\startcontents[appendix]
\section*{Table of Contents}
\printcontents[appendix]{}{1}{}
\vspace{1em}

\newpage

\section{Linear Temporal Logic}
\label{app:ltl-definitions}

We now give the basic definitions of linear temporal logic.
\subsection{Syntax}
\label{app:ltl-syntax}
We define the syntax of full linear temporal logic, $\gtl{\past, \prev, \since, \until}$, over an alphabet $\alphabet$.
The set of \defn{formulas} over $\alphabet$ is defined inductively:
\begin{itemize}[leftmargin=1.5em,labelwidth=1em,labelsep=0.5em]
  \item \textbf{Constants:} $\true$ and $\false$ are formulas.
  \item \textbf{Atomic predicates:} $\atom_\syma$ is a formula for  $\syma\in\alphabet$.
  \item \textbf{Boolean connectives:} if $\tlf_1$ and $\tlf_2$ are formulas, then so are $\lnot \tlf_1$, $\tlf_1 \land \tlf_2$, and $\tlf_1 \lor \tlf_2$.
  \item \textbf{Temporal operators:} if $\tlf$ and $\tlf'$ are formulas, then so are $\prev \tlf$ (``yesterday''), $\past \tlf$ (``past''), $\tlf \since \tlf'$ (``since''), and $\tlf \until \tlf'$ (``until'').
\end{itemize}

\subsection{Operator Depth}
\label{app:operator-depth}
The \defn{operator depth} of a formula $\fof$, denoted $\od(\fof)$, is defined inductively:
\begin{equation}
\od(\fof) \defeq
\begin{cases}
0 & \text{if } \fof \in \{\true,\, \false,\, \atom_\syma\} \\[3pt]
\od(\tlf) & \text{if } \fof = \lnot \tlf \\[3pt]
\max\big(\od(\tlf_1),\, \od(\tlf_2)\big) & \text{if } \fof = \tlf_1 \circ \tlf_2,\; \circ \in \{\land, \lor\} \\[3pt]
1 + \od(\tlf) & \text{if } \fof = \circ\, \tlf,\; \circ \in \{\prev, \past\} \\[3pt]
1 + \max\big(\od(\tlf_1),\, \od(\tlf_2)\big) & \text{if } \fof = \tlf_1 \circ \tlf_2,\; \circ \in \{\since, \until\}
\end{cases}
\end{equation}

\subsection{Semantics}
\label{app:ltl-semantics}
Given a string $\str$ and a position $n$, the semantics for the satisfaction relation $\str, n \models \tlf$ is defined as:
\begin{itemize}[leftmargin=1.5em,labelwidth=1em,labelsep=0.5em]
  \item $\str,n\models \atom_\syma$ iff $1 \le n \le |\str|$ and the $n^{\text{th}}$ token of $\str$ is $\syma$;
  \item $\str,n\models \tlf_1 \lor \tlf_2$  iff  $\str,n\models \tlf_1$ or $\str,n\models \tlf_2$;
  \item $\str,n\models \tlf_1 \land \tlf_2$  iff  $\str,n\models \tlf_1$ and $\str,n\models \tlf_2$;
  \item $\str,n\models \lnot \tlf$  iff  $\str,n \not\models \tlf$;
  \item $\str,n\models \prev \tlf$  iff  $\str, n-1\models \tlf$;
  \item $\str,n\models \past \tlf$  iff  there exists $m$ such that $m<n$ and $\str, m\models \tlf$;
  \item $\str,n\models \tlf_1\since\tlf_2$ iff there exists $m$ such that $m<n$, $\str,m\models \tlf_2$, and for all $i$ such that $m<i<n$, $\str,i\models \tlf_1$;
  \item $\str,n\models \tlf_1 \until \tlf_2$ iff there exists $m$ such that $m>n$, $\str,m\models\tlf_2$, and for all $i$ such that $n<i<m$, $\str, i\models \tlf_1$.
\end{itemize}
A string $\str$ of length $\length$ \defn{satisfies} a formula $\tlf$, written $\str\models\tlf$, if $\tlf$ holds at a position outside the string, such as position $0$ or $\length+1$. In this work, we evaluate formulas at position $\length+1$, just beyond the end of the string, since we consider only past-time operators such as $\prev$ and $\past$, i.e.,
\begin{equation}
\str\models\tlf \quad \text{iff} \quad \str,\length+1\models \tlf.
\end{equation}
The language \defn{defined} by a formula $\tlf$ is the set of all strings that satisfy it, i.e., the set
\begin{equation}
\langfun(\tlf) \defeq \{\str \in \kleene{\alphabet} \mid \str \models \tlf\}.
\end{equation}

\section{Additional Background}
\label{app:background}

We collect background on formal language classes, syntactic semigroups, and deterministic finite automata, which we use to characterize transformers under different attention patterns.

\subsection{Formal Languages}

\paragraph{Regular Languages.}
A \defn{regular expression} is a declarative description of a language over $\alphabet$. It is defined inductively as follows:
\begin{itemize}[leftmargin=1.5em, labelwidth=1em, labelsep=0.5em, listparindent=\parindent, itemsep=\parskip, parsep=0pt]
    \item $\emptyset$ and each $\syma \in \alphabet$ are regular expressions;
    \item if $\alpha$ and $\beta$ are regular expressions, then so are the union $\alpha + \beta$, the concatenation $\alpha\beta$, the Kleene star $\alpha^*$, and the complement $\setcomplement{\alpha}$.
\end{itemize}
A language is \defn{regular} if and only if it can be described by a regular expression \citep{Kleene1956}. A regular language is \defn{star-free} if it can be described by a regular expression that does not use the Kleene star \citep{mcnaughton1971counter}. For example, $\kleene{\alphabet}$ is star-free, since it can be described by the regular expression $\setcomplement{\emptyset}$.

\paragraph{Monomials.}
A \defn{monomial} over $\alphabet$ is a language of the form $\alphabet_0^* \syma_1 \alphabet_1^* \cdots \syma_n \alphabet_n^*$, where $\syma_1,\ldots,\syma_n \in \alphabet$ and $\alphabet_0,\alphabet_1,\ldots,\alphabet_n \subseteq \alphabet$.
A monomial is called
\begin{itemize}[leftmargin=1.5em, labelwidth=1em, labelsep=0.5em, listparindent=\parindent, itemsep=\parskip, parsep=0pt]
    \item \defn{left-deterministic} if for every $k \in \{1,\ldots,n\}$, $\syma_k \notin \alphabet_{k-1}$;
    \item \defn{right-deterministic} if for every $k \in \{1,\ldots,n\}$, $\syma_k \notin \alphabet_k$;
    \item \defn{unambiguous} if every string $\str$ in the monomial admits at most one factorization $\str=\str_0\syma_1\str_1\cdots\syma_n\str_n$ with $\str_k \in \kleene{\alphabet_k}$ for all $k \in \{0,\ldots,n\}$.
\end{itemize}

\paragraph{Polynomials.}
A \defn{polynomial} is a finite union of monomials. It is called left-deterministic (resp.\ right-deterministic or unambiguous) if it is a finite disjoint union of left-deterministic (resp.\ right-deterministic or unambiguous) monomials.

\subsection{Syntactic Monoids and Semigroups}

\paragraph{Semigroups and Monoids.}
A \defn{semigroup} is a set endowed with an associative binary operation.
A \defn{monoid} is a semigroup with an identity element.

\paragraph{$\gR$-trivial Monoids.}
Recall Green's $\gR$ relation on a monoid $\monoid$ \citep{Green1951}.
For $s,t \in \monoid$, we write $s \,\gR\, t$ if they generate the same
principal right ideal:
\begin{equation}
  s\monoid = t\monoid.
\end{equation}
A \emph{monoid} $\monoid$ is \defn{$\gR$-trivial} if $\gR$ is equality; equivalently,
for all $s,t \in \monoid$,
\begin{equation}
  s \,\gR\, t
  \quad\Longrightarrow\quad
  s=t.
\end{equation}

\paragraph{Locally $\gR$-trivial Semigroups.}
An element $\idem \in \semigroup$ is \defn{idempotent} if $\idem^2=\idem$.
A \emph{semigroup} $\semigroup$ is \defn{locally $\gR$-trivial} if, for every
idempotent $\idem \in \semigroup$, the submonoid
\begin{equation}
  \idem \semigroup \idem
  \defeq
  \{\idem s \idem \mid s \in \semigroup\}
\end{equation}
with identity element $\idem$ is $\gR$-trivial.

\paragraph{Syntactic Congruence.}
Let $\lang \subseteq \kleene{\alphabet}$ be a language. Two strings
$\strs,\strt \in \kleene{\alphabet}$ are \defn{syntactically equivalent},
denoted $\strs \equivalent \strt$, if
\begin{equation}
  \forall\, \stru,\strv \in \kleene{\alphabet}\colon\qquad
  \stru\strs\strv \in \lang
  \ \Leftrightarrow\
  \stru\strt\strv \in \lang.
\end{equation}
We write $\equivClass{\str}$ for the equivalence class of $\str$.

\paragraph{Syntactic Monoid.}
The quotient monoid 
\begin{equation}
\kleene{\alphabet}/\equivalent,
\end{equation}
whose elements are
equivalence classes and whose operation is induced by concatenation, is called
the \defn{syntactic monoid} of $\lang$.

\paragraph{Syntactic Semigroup.}
Let $\alphabet^+$ denote the set of nonempty strings over $\alphabet$.
Restricting the syntactic congruence to $\alphabet^+$ yields the quotient
semigroup
\begin{equation}
  \alphabet^+/\equivalent.
\end{equation}
This quotient is called the \defn{syntactic semigroup} of $\lang$.

\subsection{Deterministic Finite Automata}

\paragraph{Definition.}
A deterministic finite automaton (DFA) is a $5$-tuple $\automaton=\dfatuple$, where
\begin{itemize}[leftmargin=1.5em, labelwidth=1em, labelsep=0.5em, listparindent=\parindent, itemsep=\parskip, parsep=0pt]
  \item $\alphabet$ is an alphabet;
  \item $\states$ is a finite set of states;
  \item $\qinit \in \states$ is the initial state;
  \item $\final \subseteq \states$ is the set of accepting states;
  \item $\trans\colon \states \times \alphabet \to \states$ is the transition function.
\end{itemize}
We extend $\trans$ to $\states\times\kleene{\alphabet}$ recursively by
\begin{equation}
  \trans(\stateq,\varepsilon)=\stateq
\end{equation}
for every $\stateq\in\states$, and
\begin{equation}
  \trans(\stateq,\str\syma)=\trans(\trans(\stateq,\str),\syma)
\end{equation}
for every $\stateq\in\states$, $\str\in\kleene{\alphabet}$, and $\syma\in\alphabet$.

\paragraph{Acceptance.}
A DFA $\automaton$ accepts a string $\str=\sym_1\cdots\sym_\length \in \kleene{\alphabet}$ if the unique run $\stateq_0,\ldots,\stateq_\length$ defined by $\stateq_0=\qinit$ and $\stateq_{n+1}=\trans(\stateq_n,\sym_{n+1})$ for $n=0,\ldots,\length-1$ ends in an accepting state, i.e., $\stateq_\length \in \final$.

\paragraph{Minimal DFA.}
A DFA $\automaton$ \defn{recognizes} a language $\lang\subseteq\kleene{\alphabet}$ if it accepts exactly the strings in $\lang$.
A DFA recognizing $\lang$ is \defn{minimal} if it has the smallest number of states among all DFAs recognizing $\lang$. 

\paragraph{Partially Ordered DFAs.}
A DFA is \defn{partially ordered} if there exists a partial order $\preceq$ on $\states$ such that for every $\stateq \in \states$ and $\syma \in \alphabet$, we have $\stateq \preceq \trans(\stateq,\syma)$. Equivalently, the state-transition graph has no cycles other than self-loops.

\paragraph{Transformation.}
For each non-empty string $\str\in\alphabet^+$, the extended transition function induces a transformation
$t_\str\colon \states\to\states$ defined by
\begin{equation}
    \stateq t_\str \defeq \trans(\stateq,\str).
\end{equation}
A permutation of $\states$ is a bijection from $\states$ to itself.
A transformation $t\colon \states\to\states$ is \defn{permutational} if there exists a subset
$P\subseteq \states$ with $|P|\ge 2$ such that the restriction
\begin{equation}
  t|_P \colon P\to P
\end{equation}
is a permutation of $P$.

\paragraph{Configuration Containment.}
Let $\automaton=(\alphabet,\states,\qinit,\final,\trans)$ be a DFA.
Define $G(\automaton)$ to be the directed labeled graph whose vertex set is $\states$
and whose edge set consists of all triples $(\stateq,\str,\stateq')$,
such that $\stateq,\stateq'\in\states$, $\str\in\alphabet^+$, and $\trans(\stateq,\str)=\stateq'$.
Because $\states$ is finite but $\alphabet^+$ is infinite, the graph $G(\automaton)$
has finitely many vertices and, in general, infinitely many labeled edges.
If a directed graph whose edges are labeled by non-empty strings is a labeled subgraph of
$G(\automaton)$, then we say that $\automaton$ contains it as a configuration.

\section{Proofs}
\label{app:proofs}

\subsection{Proof of \cref{thm:ytl-definite}}
\label{app:proof-ytl-definite}

\ytldefinite*

\begin{proof}
(\ref{item:ytl}$\Longrightarrow$\ref{item:definite})
Let $\tlf$ be a $\ytl$ formula, and let $r=\od(\tlf)$.
We prove, by structural induction on $\tlf$, the following stronger claim:

\medskip
\noindent
$(\ast)$ For every string $\str=\sym_1\cdots\sym_\length$ and every position $n$,
the truth of $\str,n\models\tlf$ depends only on the factor
\begin{equation}
  \sym_{\max(1,n-r)}\cdots\sym_{\min(n,\length)}.
\end{equation}

\paragraph{Base case: $\od(\tlf)=0$.}
Then $\tlf$ is a Boolean combination of constants and atomic formulas.
Each atomic formula $\atom_{\syma}$ depends only on the token at position $n$
(if $1\le n\le \length$), and is false otherwise.
Hence the truth of $\str,n\models\tlf$ depends only on the token at position $n$.

\paragraph{Induction step.}
Suppose $\tlf=\prev\fof$ and $\od(\fof)=r-1$.
Then
\begin{equation}
  \str,n\models \prev\fof
  \quad\Longleftrightarrow\quad
  \str,n-1\models\fof.
\end{equation}
By the induction hypothesis, the truth of $\str,n-1\models\fof$ depends only on
\begin{equation}
  \sym_{\max(1,n-r)}\cdots\sym_{\min(n-1,\length)}.
\end{equation}
Therefore $\str,n\models\prev\fof$ depends only on
\begin{equation}
  \sym_{\max(1,n-r)}\cdots\sym_{\min(n,\length)}.
\end{equation}

Boolean connectives preserve the claim, so $(\ast)$ follows for all $\ytl$ formulas.

Now let $n=\length+1$. Since $\min(\length+1,\length)=\length$, the truth of
$\str,\length+1\models\tlf$ depends only on
\begin{equation}
  \sym_{\max(1,\length-r+1)}\cdots\sym_\length,
\end{equation}
that is, only on the suffix of $\str$ of length at most $r$.
Therefore $\langfun(\tlf)$ is definite.

(\ref{item:definite}$\Longrightarrow$\ref{item:ytl})
Suppose that $\lang$ is definite. Then $\lang$ is $m$-definite for some $m\ge 0$.
By definition, $\lang$ is a Boolean combination of languages of the form
\begin{equation}
  \{\str\in\kleene{\alphabet}\mid \suffix_m(\str)=\stru\},
\end{equation}
where $\stru\in\alphabet^m$.

We show that each such basic language is definable in $\ytl$.
For each string $\stru\in\kleene{\alphabet}$, define a formula $\tlf_{\stru}$ inductively by
\begin{equation}
  \tlf_{\varepsilon}\defeq \true,
\end{equation}
and, for $\strv\in\kleene{\alphabet}$ and $\syma\in\alphabet$,
\begin{equation}
  \tlf_{\strv\syma}\defeq \prev(\atom_{\syma}\land \tlf_{\strv}).
\end{equation}

We prove by induction on $|\stru|$ that, for every string $\str=\sym_1\cdots\sym_\length$ and every position $n\ge |\stru|+1$,
\begin{equation}
  \str,n\models\tlf_{\stru}
  \quad\Longleftrightarrow\quad
  \sym_{n-|\stru|}\cdots\sym_{n-1}=\stru.
\end{equation}

\paragraph{Base case: $|\stru|=0$.}
Then $\stru=\varepsilon$, and by definition $\tlf_{\varepsilon}\defeq\true$.
Hence $\str,n\models\tlf_{\varepsilon}$ for every $\str$ and every $n$.
Since the empty factor is $\varepsilon$, the claim follows.

\paragraph{Induction step.}
Suppose $\stru=\strv\syma$ with $\strv\in\kleene{\alphabet}$ and $\syma\in\alphabet$.
By definition,
\begin{equation}
  \tlf_{\strv\syma}\defeq \prev(\atom_{\syma}\land\tlf_{\strv}).
\end{equation}
Let $n\ge |\stru|+1$. Then
\begin{equation}
  \str,n\models\tlf_{\strv\syma}\Longleftrightarrow \str,n-1\models \atom_{\syma}\land\tlf_{\strv}
\end{equation}
The first conjunct says that the token at position $n-1$ is $\syma$.
Since $|\strv|=|\stru|-1$, we have $n-1\ge |\strv|+1$, so by the induction hypothesis applied at position $n-1$,
\begin{equation}
  \str,n-1\models\tlf_{\strv}
  \quad\Longleftrightarrow\quad
  \sym_{n-|\stru|}\cdots\sym_{n-2}=\strv.
\end{equation}
Therefore,
\begin{equation}
  \str,n\models\tlf_{\stru}
  \quad\Longleftrightarrow\quad
  \sym_{n-|\stru|}\cdots\sym_{n-1}=\strv\syma=\stru.
\end{equation}
This completes the induction.

Applying the claim at position $\length+1$, we obtain
\begin{align}
  \str\models\tlf_{\stru}
  \quad&\Longleftrightarrow\quad
  \str,\length+1\models\tlf_{\stru} \\
  &\Longleftrightarrow\quad
  \sym_{\length-m+1}\cdots\sym_{\length}=\stru \\
  &\Longleftrightarrow\quad
  \suffix_m(\str)=\stru.
\end{align}

Because $\ytl$ is closed under Boolean connectives, every Boolean combination of such languages is also definable in $\ytl$.
Therefore $\lang$ is definable in $\ytl$.

The equivalence \ref{item:definite}$\Longleftrightarrow$\ref{item:semigroup} can be found in \citet{eilenberg1974automata} and \citet{10.5555/3173440.3173443}, and the equivalence \ref{item:definite}$\Longleftrightarrow$\ref{item:dfa_non_permutation} can be found in \citet{Perles1963TheTO} and \citet{10.5555/3173440.3173443}.
\end{proof}

\subsection{Proof of \cref{thm:pptl}}
\label{app:proof-pptl}

\characterization*

\begin{proof}
\citet{COHEN1993271} introduce a restricted temporal logic $\rtl$ with operators $\pasteq$ and $\prev$, where
\begin{equation}
  \str,n \models \pasteq \tlf
  \quad\text{iff}\quad
  \exists\, m \leq n \ \text{such that}\ \str,m \models \tlf .
\end{equation}
They show \citep[Proposition~4.1]{COHEN1993271} that a regular language is definable in $\rtl$ if and only if its syntactic semigroup is locally $\gR$-trivial; equivalently, its minimal DFA avoids the configuration in \cref{fig:forbidden}.
It therefore remains to relate $\rtl$ and $\yptl$. In the presence of $\prev$, the operators $\past$ and $\pasteq$ are mutually definable:
\begin{subequations}
\begin{align}
  \past \tlf &\equiv \prev\,\pasteq \tlf, \\
  \pasteq \tlf &\equiv \past \tlf \lor \tlf.
\end{align}
\end{subequations}
Thus, $\rtl$ and $\yptl$ define the same class of languages, and the equivalences follow.
\end{proof}

\subsection{Proof of \cref{thm:localtest}}
\label{app:proof-localtest}

\localtest*

\begin{proof}
Fix $m>0$. It suffices to show that each of the following basic languages is definable in $\yptl$:
\begin{subequations}
\begin{align}
  &\{\str\in\kleene{\alphabet}\mid \stru\in\factor_m(\str)\},
  &&\stru\in\alphabet^m, \\
  &\{\str\in\kleene{\alphabet}\mid \prefix_{m-1}(\str)=\stru\},
  &&\stru\in\alphabet^{m-1}, \\
  &\{\str\in\kleene{\alphabet}\mid \suffix_{m-1}(\str)=\stru\},
  &&\stru\in\alphabet^{m-1}.
\end{align}
\end{subequations}

\paragraph{Case 1: $m$-factors.}
Let $\stru=\sym_1\cdots\sym_m\in\alphabet^m$. Then
\begin{equation}
  \past\bigl(
    \atom_{\sym_m}\land
    \prev(\atom_{\sym_{m-1}}\land
    \cdots\land
    \prev(\atom_{\sym_1}))
  \bigr)
\end{equation}
defines the language $\{\str\in\kleene{\alphabet}\mid \stru\in\factor_m(\str)\}$.

\paragraph{Case 2: $(m-1)$-prefixes.}
Let $\stru=\sym_1\cdots\sym_{m-1}\in\alphabet^{m-1}$. If $m=1$, then this language is all of $\kleene{\alphabet}$, and is defined by $\true$. If $m>1$, then
\begin{equation}
  \past\bigl(
    \atom_{\sym_{m-1}}\land
    \prev(\atom_{\sym_{m-2}}\land
    \cdots\land
    \prev(\atom_{\sym_1}\land \lnot\past\true))
  \bigr)
\end{equation}
defines the language $\{\str\in\kleene{\alphabet}\mid \prefix_{m-1}(\str)=\stru\}$.

\paragraph{Case 3: $(m-1)$-suffixes.}
Let $\stru=\sym_1\cdots\sym_{m-1}\in\alphabet^{m-1}$. If $m=1$, then this language is all of $\kleene{\alphabet}$, and is defined by $\true$. If $m>1$, then
\begin{equation}
  \prev\bigl(
    \atom_{\sym_{m-1}}\land
    \prev(\atom_{\sym_{m-2}}\land
    \cdots\land
    \prev(\atom_{\sym_1}))
  \bigr)
\end{equation}
defines the language $\{\str\in\kleene{\alphabet}\mid \suffix_{m-1}(\str)=\stru\}$.

Because $\yptl$ is closed under Boolean connectives, every locally testable language is definable in $\yptl$.
\end{proof}

\subsection{Proof of \cref{thm:ptl_ytlk_incomparable}}
\label{app:proof-ptl-ytlk-incomparable}

\ptlytlkincomparable*

\begin{proof}
($\Longrightarrow$) The language $\syma\kleene{\alphabet}$ is definable in $\ptl$ but not in $\ytl[k]$, by the same argument as in \cref{prop:ptl_not_ytl}.

($\Longleftarrow$)
Consider the language $\lang_k \defeq \bigcup_{i=0}^{k-1}\kleene{\alphabet}\syma\alphabet^i$,
the set of strings that contain an $\syma$ among the last $k$ positions. This language is definable in $\ytl[k]$ by $\prev[k]\atom_\syma$.
However,  $\lang_k$ is not a left-deterministic polynomial, and therefore is not definable in $\ptl$.
See \cref{app:background} for the necessary definitions.
\end{proof}

\subsection{Proof of \cref{cor:yptlk_strictly_more_expressive}}
\label{app:proof-yptlk-strictly}

\yptlkstrictly*

\begin{proof}
The inclusions $\ptl \subseteq \yptl[k]$ and $\ytl[k] \subseteq \yptl[k]$ hold because $\yptl[k]$ has access to both operators $\past$ and $\prev[k]$. Strictness follows from \cref{thm:ptl_ytlk_incomparable}---the language $\lang_k$ is definable in $\ytl[k]$ but not in $\ptl$, so $\ptl \neq \yptl[k]$.
Moreover, $\syma\kleene{\alphabet}$ is definable in $\ptl$ but not in $\ytl[k]$, so $\ytl[k] \neq \yptl[k]$.
\end{proof}

\subsection{Proof of \cref{prop:ytlk_subset_ytl}}
\label{app:proof-ytlk-subset-ytl}

As a preparation, we prove the following lemma. 

\begin{lemma}
  \label{lem:last-not-ytlk-yptlk}
  Assume $\alphabet$ contains distinct tokens $\syma,\symb$.
  Then $\kleene{\alphabet}\syma$ is not definable in $\ytl[k]$ or $\yptl[k]$
  for any $k>1$.
\end{lemma}

\begin{proof}
Since $\ytl[k]$ is a subfragment of $\yptl[k]$, it suffices to show that
$\kleene{\alphabet}\syma$ is not definable in $\yptl[k]$.
Let $\tlf$ be a formula in $\yptl[k]$ of operator depth $r$.
If $r=0$, then $\tlf$ is a Boolean combination of constants and atomic formulas.
Since atomic formulas are always false at the end position, such a formula has the
same truth value on every string, and therefore does not define
$\kleene{\alphabet}\syma$.
So assume $r>0$.
Define
\begin{equation}
  \str \defeq (\syma\symb^{k-1})^r \syma,
  \qquad
  \str' \defeq (\syma\symb^{k-1})^r .
\end{equation}
Then
\begin{equation}
  \str \in \kleene{\alphabet}\syma,
  \qquad
  \str' \notin \kleene{\alphabet}\syma.
\end{equation}
Let $\length \defeq |\str| = kr+1$.
Now, for each $s$ with $0 \le s \le r$, and positions
$n \in \{1,\ldots,\length+1\}$ and $n' \in \{1,\ldots,\length\}$,
we say that $(n,n')$ is \defn{$s$-close} if one of the following holds:
\begin{enumerate}[leftmargin=2.25em,itemsep=0.25ex]
  \item $n=n'$;
  \item $s\ge 1$, $n=\length+1$, and $n'=\length$;
  \item $n>ks$ and $n'=n-k$.
\end{enumerate}
We claim that for every $0 \le s \le r$, every formula $\fof$ in $\yptl[k]$
of operator depth at most $s$, and every $s$-close pair $(n,n')$, we have
\begin{equation}
  \str,n \models \fof
  \quad\Longleftrightarrow\quad
  \str',n' \models \fof .
\end{equation}

We prove this claim by induction on $s$.

\paragraph{Base case: $s=0$.}
Then $\fof$ is a Boolean combination of constants and atomic formulas.
If $n=n'$, the claim is immediate.
If $n'=n-k$, then the token at position $n$ of $\str$ is the same as the token
at position $n'$ of $\str'$, because both strings consist of repetitions of the
block $\syma\symb^{k-1}$, except that $\str$ has one extra final $\syma$, which
matches the $\syma$ at position $\length-k$ in $\str'$.
Hence $\str,n$ and $\str',n'$ satisfy the same atomic formulas, and therefore
the same Boolean combinations of atomic formulas.

\paragraph{Induction step.}
Assume the claim holds for $s-1$, and let $\fof$ have operator depth at most $s$.

The Boolean cases are immediate, so it remains to consider the temporal operators.

\medskip
\noindent
\textbf{Case 1:} $\fof=\past \psi$, where $\psi$ has operator depth at most $s-1$.

Suppose $\str,n\models \past\psi$. Then there exists $t<n$ such that
$\str,t\models \psi$.
We choose $t'<n'$ as follows:
\begin{enumerate}[leftmargin=2.25em,itemsep=0.25ex]
  \item If $n=n'$, let $t'=t$.
  \item If $n=\length+1$ and $n'=\length$, then:
  \begin{itemize}[leftmargin=1.75em,itemsep=0.15ex,topsep=0.15ex]
    \item if $t<\length$, let $t'=t$;
    \item if $t=\length$, let $t'=\length-k$.
  \end{itemize}
  \item If $n>ks$ and $n'=n-k$, then:
  \begin{itemize}[leftmargin=1.75em,itemsep=0.15ex,topsep=0.15ex]
    \item if $t<n-k$, let $t'=t$;
    \item if $n-k\le t<n$, let $t'=t-k$.
  \end{itemize}
\end{enumerate}
In each case, $t'<n'$ and $(t,t')$ is $(s-1)$-close.
By the induction hypothesis,
\begin{equation}
  \str,t\models\psi
  \quad\Longleftrightarrow\quad
  \str',t'\models\psi,
\end{equation}
so $\str',n'\models \past\psi$.
The converse implication is analogous.

\medskip
\noindent
\textbf{Case 2:} $\fof=\prev[k]\psi$, where $\psi$ has operator depth at most $s-1$.

Suppose $\str,n\models \prev[k]\psi$. Then there exists $t$ such that
$n-k \le t < n$ and $\str,t\models\psi$.
We choose $t'$ as follows:
\begin{enumerate}[leftmargin=2.25em,itemsep=0.25ex]
  \item If $n=n'$, let $t'=t$.
  \item If $n=\length+1$ and $n'=\length$, then:
  \begin{itemize}[leftmargin=1.5em,itemsep=0.25ex]
    \item if $t<\length$, let $t'=t$;
    \item if $t=\length$, let $t'=\length-k$.
  \end{itemize}
  \item If $n>ks$ and $n'=n-k$, let $t'=t-k$.
\end{enumerate}
In each case,
\begin{equation}
  n'-k \le t' < n',
\end{equation}
and $(t,t')$ is $(s-1)$-close.
By the induction hypothesis,
\begin{equation}
  \str,t\models\psi
  \quad\Longleftrightarrow\quad
  \str',t'\models\psi,
\end{equation}
so $\str',n'\models \prev[k]\psi$.
Again, the converse implication is analogous.

This completes the induction.

Finally, the pair $(\length+1,\length)$ is $r$-close.
Hence
\begin{equation}
  \str \models \tlf
  \quad\Longleftrightarrow\quad
  \str,\length+1 \models \tlf
  \quad\Longleftrightarrow\quad
  \str',\length \models \tlf
  \quad\Longleftrightarrow\quad
  \str' \models \tlf.
\end{equation}
Thus no formula in $\yptl[k]$ can distinguish $\str$ from $\str'$.
Since $\str \in \kleene{\alphabet}\syma$ and
$\str' \notin \kleene{\alphabet}\syma$,
the language $\kleene{\alphabet}\syma$ is not definable in $\yptl[k]$.
Therefore it is not definable in $\ytl[k]$ either.
\end{proof}

\ytlksubsetytl*

\begin{proof}
(\textbf{Inclusion.}) Any formula of the form $\prev[k]\tlf$ can be expanded into a finite disjunction of $\prev$-steps by \cref{eq:prevk_def}.
Thus, every $\ytl[k]$ formula translates into an equivalent $\ytl$ formula. Hence $\ytl[k] \subseteq \ytl$.

(\textbf{Strictness.}) Consider the language $\kleene{\alphabet}\syma$, the set of strings whose last token is $\syma$. This language is definable in $\ytl$ by $\prev\atom_\syma$, but is not definable in $\ytl[k]$ for any $k>1$ by \cref{lem:last-not-ytlk-yptlk}. Therefore $\ytl[k] \subsetneq \ytl$ for $k>1$.
\end{proof}

\subsection{Proof of \cref{prop:yptlk_subset_yptl}}
\label{app:proof-yptlk-subset-yptl}

\yptlksubsetyptl*

\begin{proof}
The inclusion $\yptl[k] \subseteq \yptl$ follows from \cref{eq:prevk_def}: every occurrence of $\prev[k]$ can be rewritten using $\prev$, which is available in $\yptl$.
For strictness, consider again the language $\kleene{\alphabet}\syma$. This language is definable in $\yptl$ but not definable in $\yptl[k]$ for any $k>1$ by \cref{lem:last-not-ytlk-yptlk}. Therefore $\yptl[k] \subsetneq \yptl$ for $k>1$.
\end{proof}

\subsection{Proof of \cref{prop:modular}}
\label{app:proof-modular}

\modularproposition*

\begin{proof}
It suffices to show that $\prev$ can be expressed using $\prev[k]$ together with modular predicates. Indeed, for any $m \ge k$, we can write
\begin{equation}
  \prev \tlf = 
  \bigvee_{i=1}^{m}
  \Bigl(
    \mod[m][i]
    \land
    \prev[k](\mod[m][i-1] \land \tlf)
  \Bigr),
\end{equation}
where $i-1$ is understood modulo $m$.
To see this, fix a position $n$. Exactly one disjunct applies, namely the one with
$n \equiv i \pmod m$. The formula
\begin{equation}
  \prev[k](\mod[m][i-1] \land \tlf)
\end{equation}
then asks whether $\tlf$ holds at some position among the previous $k$ positions whose
residue modulo $m$ is $i-1$.
Now $n-1$ is such a position. 
Moreover, it is the only such position among the previous
$k$ positions: if $n-j \equiv n-1 \pmod m$ for some $1 \le j \le k$, then
$j-1 \equiv 0 \pmod m$. Since $0 \le j-1 < k \le m$, this implies $j=1$.
Hence the unique position in the previous $k$ positions with residue class $i-1$ is
the immediate predecessor $n-1$.
\end{proof}

\section{Additional Experimental Results}

This section provides additional details of the experimental setup and supplementary results for the formal language recognition experiments in \cref{sec:formal_experiments}.

\subsection{Datasets and Artifacts}
\label{app:data-artifacts}
For the formal language recognition experiments, we generate synthetic datasets. These data consist entirely of artificial sequences over small finite alphabets. 
For natural language experiments, we use the public WikiText-2 benchmark. We do not collect or annotate new natural language data. Any risks related to personally identifying information or offensive content are therefore inherited from the source dataset rather than introduced by our work.
Our implementation is based in part on the Hugging Face codebase \citep{wolf-etal-2020-transformers}, but all models in our experiments are trained from scratch and no pre-trained weights are used.

\subsection{Experimental Setup}
\label{app:setup_formal}
Our setup for formal language recognition follows \citet{li2025} and \citet{deletang2023neural}. We use a $5$-layer transformer with model dimension $64$ and $8$ attention heads. Training strings are sampled at lengths up to $40$, while test strings range from length $41$ to $500$. Models are trained for $1{,}000{,}000$ steps using Adam \citep{kingma2014adam} with default parameters and a batch size of $128$. At evaluation time, we sample $512$ test strings at each length.
Each configuration is run with five random seeds and three learning rates: $1\times10^{-4}$, $3\times10^{-4}$, and $5\times10^{-4}$. All experiments are conducted on a single GPU with $24$\,GB of memory, and each run takes roughly one hour.
Our codebase is adapted from \citet{deletang2023neural} and \citet{li2025}.\looseness=-1

\subsection{Full Numerical Results}
\label{app:results}

We report the full numerical results of the formal language recognition experiments in \cref{tab:formal_languages_acc_nope,tab:formal_languages_acc_rope,tab:formal_languages_acc_sipe,tab:formal_languages_len_nope,tab:formal_languages_len_rope,tab:formal_languages_len_sipe}.

\begin{table*}[t]
\centering
\scriptsize
\setlength{\tabcolsep}{3pt}
\renewcommand{\arraystretch}{1.08}
\begin{tabular}{lllccccccc}
\toprule
\multicolumn{3}{c}{} & \multicolumn{3}{c}{Local} & \multicolumn{3}{c}{Hybrid} & \multicolumn{1}{c}{Global} \\
\cmidrule(lr){4-6}\cmidrule(lr){7-9}\cmidrule(lr){10-10}
\multicolumn{3}{c}{} & $k=1$ & $k=2$ & $k=4$ & $k=1$ & $k=2$ & $k=4$ & $-$ \\
\midrule
\multirow{4}{*}{$\ltl$} & \multirow{2}{*}{$\kleene{(\syma\allowbreak\kleene{(\syma\symb)}\allowbreak\symb)}$} & mean$\pm$std & 75.2 $\pm$ 0.0 & 94.7 $\pm$ 0.1 & 98.4 $\pm$ 4.1 & 89.7 $\pm$ 8.9 & 87.4 $\pm$ 12.6 & 83.5 $\pm$ 10.7 & 66.7 $\pm$ 10.3 \\
 &  & max & 75.2 & 94.7 & 99.6 & 99.5 & 98.7 & 92.0 & 99.8 \\
 & \multirow{2}{*}{$\kleene{\{\syma,\symb,\symd\}}\allowbreak\syma\allowbreak\kleene{\{\symc,\symd\}}$} & mean$\pm$std & 78.4 $\pm$ 0.0 & 91.8 $\pm$ 0.0 & 99.1 $\pm$ 0.0 & 98.6 $\pm$ 2.8 & 99.9 $\pm$ 0.1 & 90.4 $\pm$ 4.9 & 76.1 $\pm$ 11.9 \\
 &  & max & 78.4 & 91.8 & 99.2 & 100.0 & 100.0 & 96.3 & 91.2 \\
\midrule
\multirow{4}{*}{$\yptl$} & \multirow{2}{*}{$\kleene{\alphabet}\allowbreak\syma\allowbreak\symb\allowbreak\kleene{\alphabet}$} & mean$\pm$std & 71.2 $\pm$ 1.0 & 85.1 $\pm$ 0.0 & 96.2 $\pm$ 0.0 & 99.7 $\pm$ 0.8 & 100.0 $\pm$ 0.0 & 71.9 $\pm$ 11.9 & 56.9 $\pm$ 2.2 \\
 &  & max & 72.5 & 85.2 & 96.2 & \textbf{100.0} & \textbf{100.0} & 91.3 & 61.9 \\
 & \multirow{2}{*}{$\kleene{(\syma\symb)}$} & mean$\pm$std & 52.0 $\pm$ 0.0 & 54.1 $\pm$ 0.1 & 57.3 $\pm$ 2.5 & 99.8 $\pm$ 0.4 & 95.1 $\pm$ 8.9 & 94.6 $\pm$ 12.2 & 75.2 $\pm$ 13.1 \\
 &  & max & 52.0 & 54.1 & 58.5 & \textbf{100.0} & \textbf{100.0} & \textbf{100.0} & 97.2 \\
\midrule
\multirow{4}{*}{$\ptl$} & \multirow{2}{*}{$\kleene{\alphabet}\allowbreak\syma\allowbreak\kleene{\alphabet}\allowbreak\symb\allowbreak\kleene{\alphabet}$} & mean$\pm$std & 71.6 $\pm$ 0.1 & 74.3 $\pm$ 0.2 & 74.9 $\pm$ 0.0 & 100.0 $\pm$ 0.1 & 99.9 $\pm$ 0.2 & 100.0 $\pm$ 0.1 & 99.0 $\pm$ 2.1 \\
 &  & max & 71.7 & 74.4 & 74.9 & \textbf{100.0} & \textbf{100.0} & \textbf{100.0} & \textbf{100.0} \\
 & \multirow{2}{*}{$\syma\kleene{\alphabet}$} & mean$\pm$std & 50.1 $\pm$ 0.0 & 49.9 $\pm$ 0.0 & 50.0 $\pm$ 0.1 & 100.0 $\pm$ 0.0 & 100.0 $\pm$ 0.0 & 99.2 $\pm$ 1.7 & 99.4 $\pm$ 1.4 \\
 &  & max & 50.1 & 50.1 & 50.0 & \textbf{100.0} & \textbf{100.0} & \textbf{100.0} & \textbf{100.0} \\
\midrule
\multirow{4}{*}{$\ytl$} & \multirow{2}{*}{$\kleene{\alphabet}\syma\symb$} & mean$\pm$std & \textbf{100.0} $\pm$ 0.0 & 99.8 $\pm$ 0.1 & 100.0 $\pm$ 0.0 & \textbf{100.0} $\pm$ 0.0 & 99.8 $\pm$ 0.1 & 80.7 $\pm$ 12.3 & 53.5 $\pm$ 0.8 \\
 &  & max & \textbf{100.0} & 99.8 & 100.0 & \textbf{100.0} & 100.0 & 96.9 & 54.9 \\
 & \multirow{2}{*}{$\kleene{\alphabet}\syma$} & mean$\pm$std & \textbf{100.0} $\pm$ 0.0 & 99.7 $\pm$ 0.0 & 100.0 $\pm$ 0.0 & \textbf{100.0} $\pm$ 0.0 & 99.7 $\pm$ 0.7 & 92.7 $\pm$ 8.0 & 58.4 $\pm$ 4.1 \\
 &  & max & \textbf{100.0} & 99.7 & 100.0 & \textbf{100.0} & 100.0 & 99.1 & 67.1 \\
\bottomrule
\end{tabular}
\caption{Accuracy (\%) on the formal-language tasks with $\nope$. }
\label{tab:formal_languages_acc_nope}
\end{table*}

\begin{table*}[t]
\centering
\scriptsize
\setlength{\tabcolsep}{3pt}
\renewcommand{\arraystretch}{1.08}
\begin{tabular}{lllccccccc}
\toprule
\multicolumn{3}{c}{} & \multicolumn{3}{c}{Local} & \multicolumn{3}{c}{Hybrid} & \multicolumn{1}{c}{Global} \\
\cmidrule(lr){4-6}\cmidrule(lr){7-9}\cmidrule(lr){10-10}
\multicolumn{3}{c}{} & $k=1$ & $k=2$ & $k=4$ & $k=1$ & $k=2$ & $k=4$ & $-$ \\
\midrule
\multirow{4}{*}{$\ltl$} & \multirow{2}{*}{$\kleene{(\syma\allowbreak\kleene{(\syma\symb)}\allowbreak\symb)}$} & mean$\pm$std & 75.2 $\pm$ 0.0 & 62.6 $\pm$ 16.3 & \textbf{100.0} $\pm$ 0.0 & 82.5 $\pm$ 12.3 & 76.5 $\pm$ 12.5 & 74.9 $\pm$ 14.6 & 56.1 $\pm$ 4.6 \\
 &  & max & 75.2 & 95.3 & \textbf{100.0} & 95.4 & 95.5 & 96.8 & 63.6 \\
 & \multirow{2}{*}{$\kleene{\{\syma,\symb,\symd\}}\allowbreak\syma\allowbreak\kleene{\{\symc,\symd\}}$} & mean$\pm$std & 78.4 $\pm$ 0.0 & 91.8 $\pm$ 0.1 & 99.2 $\pm$ 0.0 & 91.0 $\pm$ 6.7 & 95.9 $\pm$ 4.7 & 97.3 $\pm$ 5.6 & 58.4 $\pm$ 11.1 \\
 &  & max & 78.4 & 91.8 & 99.2 & 99.5 & 100.0 & 100.0 & 87.6 \\
\midrule
\multirow{4}{*}{$\yptl$} & \multirow{2}{*}{$\kleene{\alphabet}\allowbreak\syma\allowbreak\symb\allowbreak\kleene{\alphabet}$} & mean$\pm$std & 71.3 $\pm$ 0.9 & 85.1 $\pm$ 0.0 & 96.1 $\pm$ 0.1 & 88.3 $\pm$ 10.9 & 86.4 $\pm$ 6.1 & 89.3 $\pm$ 9.2 & 51.9 $\pm$ 0.5 \\
 &  & max & 72.5 & 85.3 & 96.2 & 100.0 & 99.6 & \textbf{100.0} & 52.6 \\
 & \multirow{2}{*}{$\kleene{(\syma\symb)}$} & mean$\pm$std & 52.0 $\pm$ 0.0 & 54.2 $\pm$ 0.0 & 57.8 $\pm$ 1.0 & 74.0 $\pm$ 10.1 & 76.5 $\pm$ 12.4 & 77.2 $\pm$ 13.2 & 53.3 $\pm$ 3.0 \\
 &  & max & 52.0 & 54.2 & 58.6 & 92.4 & \textbf{100.0} & 99.7 & 61.3 \\
\midrule
\multirow{4}{*}{$\ptl$} & \multirow{2}{*}{$\kleene{\alphabet}\allowbreak\syma\allowbreak\kleene{\alphabet}\allowbreak\symb\allowbreak\kleene{\alphabet}$} & mean$\pm$std & 71.7 $\pm$ 0.1 & 74.3 $\pm$ 0.1 & 74.9 $\pm$ 0.0 & 92.1 $\pm$ 7.8 & 97.6 $\pm$ 3.2 & 95.0 $\pm$ 7.0 & 88.0 $\pm$ 7.8 \\
 &  & max & 71.7 & 74.4 & 75.0 & 100.0 & \textbf{100.0} & \textbf{100.0} & 100.0 \\
 & \multirow{2}{*}{$\syma\kleene{\alphabet}$} & mean$\pm$std & 50.1 $\pm$ 0.0 & 49.9 $\pm$ 0.0 & 49.9 $\pm$ 0.0 & 70.1 $\pm$ 10.1 & 70.3 $\pm$ 11.9 & 77.7 $\pm$ 14.2 & 93.1 $\pm$ 8.8 \\
 &  & max & 50.1 & 50.0 & 50.0 & 96.3 & 99.7 & \textbf{100.0} & 99.9 \\
\midrule
\multirow{4}{*}{$\ytl$} & \multirow{2}{*}{$\kleene{\alphabet}\syma\symb$} & mean$\pm$std & \textbf{100.0} $\pm$ 0.0 & \textbf{100.0} $\pm$ 0.0 & \textbf{100.0} $\pm$ 0.0 & 100.0 $\pm$ 0.0 & \textbf{100.0} $\pm$ 0.0 & \textbf{100.0} $\pm$ 0.0 & 52.0 $\pm$ 1.1 \\
 &  & max & \textbf{100.0} & \textbf{100.0} & \textbf{100.0} & \textbf{100.0} & \textbf{100.0} & \textbf{100.0} & 54.6 \\
 & \multirow{2}{*}{$\kleene{\alphabet}\syma$} & mean$\pm$std & \textbf{100.0} $\pm$ 0.0 & \textbf{100.0} $\pm$ 0.0 & \textbf{100.0} $\pm$ 0.0 & \textbf{100.0} $\pm$ 0.0 & \textbf{100.0} $\pm$ 0.0 & \textbf{100.0} $\pm$ 0.0 & 53.0 $\pm$ 1.9 \\
 &  & max & \textbf{100.0} & \textbf{100.0} & \textbf{100.0} & \textbf{100.0} & \textbf{100.0} & \textbf{100.0} & 57.0 \\
\bottomrule
\end{tabular}
\caption{Accuracy (\%) on the formal-language tasks with $\rotary$. }
\label{tab:formal_languages_acc_rope}
\end{table*}

\begin{table*}[t]
\centering
\scriptsize
\setlength{\tabcolsep}{3pt}
\renewcommand{\arraystretch}{1.08}
\begin{tabular}{lllccccccc}
\toprule
\multicolumn{3}{c}{} & \multicolumn{3}{c}{Local} & \multicolumn{3}{c}{Hybrid} & \multicolumn{1}{c}{Global} \\
\cmidrule(lr){4-6}\cmidrule(lr){7-9}\cmidrule(lr){10-10}
\multicolumn{3}{c}{} & $k=1$ & $k=2$ & $k=4$ & $k=1$ & $k=2$ & $k=4$ & $-$ \\
\midrule
\multirow{4}{*}{$\ltl$} & \multirow{2}{*}{$\kleene{(\syma\allowbreak\kleene{(\syma\symb)}\allowbreak\symb)}$} & mean$\pm$std & 73.9 $\pm$ 1.9 & 74.0 $\pm$ 16.2 & 57.4 $\pm$ 5.9 & 84.5 $\pm$ 13.7 & 74.6 $\pm$ 16.2 & 59.0 $\pm$ 15.8 & 50.1 $\pm$ 0.4 \\
 &  & max & 75.2 & 94.3 & 71.2 & 99.1 & 98.8 & 97.2 & 50.7 \\
 & \multirow{2}{*}{$\kleene{\{\syma,\symb,\symd\}}\allowbreak\syma\allowbreak\kleene{\{\symc,\symd\}}$} & mean$\pm$std & 76.9 $\pm$ 2.7 & 86.3 $\pm$ 5.0 & 90.0 $\pm$ 5.8 & 92.8 $\pm$ 8.6 & 90.8 $\pm$ 10.8 & 90.9 $\pm$ 9.7 & 64.5 $\pm$ 12.2 \\
 &  & max & 78.4 & 91.6 & 98.2 & 99.8 & 100.0 & 99.9 & 87.3 \\
\midrule
\multirow{4}{*}{$\yptl$} & \multirow{2}{*}{$\kleene{\alphabet}\allowbreak\syma\allowbreak\symb\allowbreak\kleene{\alphabet}$} & mean$\pm$std & 70.6 $\pm$ 2.9 & 71.5 $\pm$ 9.0 & 76.8 $\pm$ 8.7 & 71.4 $\pm$ 12.4 & 75.0 $\pm$ 16.6 & 58.0 $\pm$ 3.4 & 51.2 $\pm$ 0.2 \\
 &  & max & 72.3 & 85.2 & 93.0 & 96.7 & 100.0 & 63.8 & 51.5 \\
 & \multirow{2}{*}{$\kleene{(\syma\symb)}$} & mean$\pm$std & 51.8 $\pm$ 0.3 & 52.1 $\pm$ 1.4 & 50.9 $\pm$ 1.1 & 86.0 $\pm$ 13.6 & 62.3 $\pm$ 13.6 & 54.1 $\pm$ 6.6 & 50.3 $\pm$ 0.2 \\
 &  & max & 52.0 & 54.1 & 54.2 & 99.8 & \textbf{100.0} & 76.1 & 51.1 \\
\midrule
\multirow{4}{*}{$\ptl$} & \multirow{2}{*}{$\kleene{\alphabet}\allowbreak\syma\allowbreak\kleene{\alphabet}\allowbreak\symb\allowbreak\kleene{\alphabet}$} & mean$\pm$std & 71.0 $\pm$ 0.5 & 73.0 $\pm$ 1.6 & 73.6 $\pm$ 2.5 & 87.4 $\pm$ 11.2 & 92.5 $\pm$ 10.5 & 96.1 $\pm$ 5.8 & 92.3 $\pm$ 9.5 \\
 &  & max & 71.7 & 74.5 & 74.9 & \textbf{100.0} & \textbf{100.0} & \textbf{100.0} & \textbf{100.0} \\
 & \multirow{2}{*}{$\syma\kleene{\alphabet}$} & mean$\pm$std & 50.0 $\pm$ 0.0 & 50.0 $\pm$ 0.0 & 50.0 $\pm$ 0.1 & 98.8 $\pm$ 2.3 & 97.8 $\pm$ 4.4 & 95.8 $\pm$ 6.7 & 98.6 $\pm$ 4.7 \\
 &  & max & 50.1 & 50.1 & 50.1 & \textbf{100.0} & \textbf{100.0} & \textbf{100.0} & \textbf{100.0} \\
\midrule
\multirow{4}{*}{$\ytl$} & \multirow{2}{*}{$\kleene{\alphabet}\syma\symb$} & mean$\pm$std & 99.5 $\pm$ 1.0 & 69.6 $\pm$ 11.6 & 60.2 $\pm$ 5.9 & 99.1 $\pm$ 1.4 & 69.9 $\pm$ 10.4 & 64.3 $\pm$ 10.7 & 50.1 $\pm$ 0.1 \\
 &  & max & \textbf{100.0} & 91.4 & 68.7 & \textbf{100.0} & 93.9 & 86.1 & 50.3 \\
 & \multirow{2}{*}{$\kleene{\alphabet}\syma$} & mean$\pm$std & 95.7 $\pm$ 3.5 & 89.6 $\pm$ 3.4 & 83.1 $\pm$ 3.0 & 98.3 $\pm$ 2.4 & 88.8 $\pm$ 2.3 & 83.3 $\pm$ 3.7 & 51.6 $\pm$ 0.3 \\
 &  & max & \textbf{100.0} & 97.7 & 87.3 & \textbf{100.0} & 92.2 & 94.2 & 52.0 \\
\bottomrule
\end{tabular}
\caption{Accuracy (\%) on the formal-language tasks with $\sincos$. }
\label{tab:formal_languages_acc_sipe}
\end{table*}

\begin{table*}[t]
\centering
\scriptsize
\setlength{\tabcolsep}{3pt}
\renewcommand{\arraystretch}{1.08}
\begin{tabular}{lllccccccc}
\toprule
\multicolumn{3}{c}{} & \multicolumn{3}{c}{Local} & \multicolumn{3}{c}{Hybrid} & \multicolumn{1}{c}{Global} \\
\cmidrule(lr){4-6}\cmidrule(lr){7-9}\cmidrule(lr){10-10}
\multicolumn{3}{c}{} & $k=1$ & $k=2$ & $k=4$ & $k=1$ & $k=2$ & $k=4$ & $-$ \\
\midrule
\multirow{4}{*}{$\ltl$} & \multirow{2}{*}{$\kleene{(\syma\allowbreak\kleene{(\syma\symb)}\allowbreak\symb)}$} & mean$\pm$std & 0.0 $\pm$ 0.0 & 0.0 $\pm$ 0.0 & 0.0 $\pm$ 0.0 & 82.4 $\pm$ 16.6 & 81.9 $\pm$ 26.2 & 75.1 $\pm$ 11.0 & 136.5 $\pm$ 57.0 \\
 &  & max & 0.0 & 0.0 & 0.0 & 126.0 & 174.0 & 104.0 & 320.0 \\
 & \multirow{2}{*}{$\kleene{\{\syma,\symb,\symd\}}\allowbreak\syma\allowbreak\kleene{\{\symc,\symd\}}$} & mean$\pm$std & 0.0 $\pm$ 0.0 & 0.0 $\pm$ 0.0 & 0.0 $\pm$ 0.0 & 73.2 $\pm$ 16.6 & 203.1 $\pm$ 102.1 & 72.3 $\pm$ 7.9 & 65.1 $\pm$ 12.2 \\
 &  & max & 0.0 & 0.0 & 0.0 & 114.0 & 455.0 & 84.0 & 100.0 \\
\midrule
\multirow{4}{*}{$\yptl$} & \multirow{2}{*}{$\kleene{\alphabet}\allowbreak\syma\allowbreak\symb\allowbreak\kleene{\alphabet}$} & mean$\pm$std & 0.0 $\pm$ 0.0 & 0.0 $\pm$ 0.0 & 0.0 $\pm$ 0.0 & 451.7 $\pm$ 86.4 & 476.4 $\pm$ 61.7 & 61.1 $\pm$ 7.7 & 61.4 $\pm$ 10.4 \\
 &  & max & 0.0 & 0.0 & 0.0 & \textbf{500.0} & \textbf{500.0} & 73.0 & 85.0 \\
 & \multirow{2}{*}{$\kleene{(\syma\symb)}$} & mean$\pm$std & 0.0 $\pm$ 0.0 & 0.0 $\pm$ 0.0 & 0.0 $\pm$ 0.0 & 436.3 $\pm$ 92.2 & 437.3 $\pm$ 105.5 & 358.4 $\pm$ 157.8 & 80.9 $\pm$ 25.4 \\
 &  & max & 0.0 & 0.0 & 0.0 & \textbf{500.0} & \textbf{500.0} & \textbf{500.0} & 132.0 \\
\midrule
\multirow{4}{*}{$\ptl$} & \multirow{2}{*}{$\kleene{\alphabet}\allowbreak\syma\allowbreak\kleene{\alphabet}\allowbreak\symb\allowbreak\kleene{\alphabet}$} & mean$\pm$std & 0.0 $\pm$ 0.0 & 0.0 $\pm$ 0.0 & 0.0 $\pm$ 0.0 & 483.9 $\pm$ 40.4 & 451.5 $\pm$ 84.2 & 452.7 $\pm$ 87.1 & 458.9 $\pm$ 75.2 \\
 &  & max & 0.0 & 0.0 & 0.0 & \textbf{500.0} & \textbf{500.0} & \textbf{500.0} & \textbf{500.0} \\
 & \multirow{2}{*}{$\syma\kleene{\alphabet}$} & mean$\pm$std & 0.0 $\pm$ 0.0 & 0.0 $\pm$ 0.0 & 0.0 $\pm$ 0.0 & 494.0 $\pm$ 15.4 & 475.0 $\pm$ 60.1 & 432.1 $\pm$ 104.5 & 359.2 $\pm$ 127.6 \\
 &  & max & 0.0 & 0.0 & 0.0 & \textbf{500.0} & \textbf{500.0} & \textbf{500.0} & \textbf{500.0} \\
\midrule
\multirow{4}{*}{$\ytl$} & \multirow{2}{*}{$\kleene{\alphabet}\syma\symb$} & mean$\pm$std & \textbf{500.0} $\pm$ 0.0 & 41.9 $\pm$ 0.3 & 68.3 $\pm$ 18.1 & \textbf{500.0} $\pm$ 0.0 & 152.4 $\pm$ 52.5 & 62.2 $\pm$ 4.9 & 50.8 $\pm$ 5.0 \\
 &  & max & \textbf{500.0} & 42.0 & 107.0 & \textbf{500.0} & 236.0 & 75.0 & 64.0 \\
 & \multirow{2}{*}{$\kleene{\alphabet}\syma$} & mean$\pm$std & \textbf{500.0} $\pm$ 0.0 & 42.0 $\pm$ 0.0 & 77.7 $\pm$ 31.5 & \textbf{500.0} $\pm$ 0.0 & 157.3 $\pm$ 70.5 & 67.0 $\pm$ 10.3 & 52.8 $\pm$ 7.0 \\
 &  & max & \textbf{500.0} & 42.0 & 172.0 & \textbf{500.0} & 334.0 & 85.0 & 74.0 \\
\bottomrule
\end{tabular}
\caption{Longest perfect length on the formal-language tasks with $\nope$. }
\label{tab:formal_languages_len_nope}
\end{table*}

\begin{table*}[t]
\centering
\scriptsize
\setlength{\tabcolsep}{3pt}
\renewcommand{\arraystretch}{1.08}
\begin{tabular}{lllccccccc}
\toprule
\multicolumn{3}{c}{} & \multicolumn{3}{c}{Local} & \multicolumn{3}{c}{Hybrid} & \multicolumn{1}{c}{Global} \\
\cmidrule(lr){4-6}\cmidrule(lr){7-9}\cmidrule(lr){10-10}
\multicolumn{3}{c}{} & $k=1$ & $k=2$ & $k=4$ & $k=1$ & $k=2$ & $k=4$ & $-$ \\
\midrule
\multirow{4}{*}{$\ltl$} & \multirow{2}{*}{$\kleene{(\syma\allowbreak\kleene{(\syma\symb)}\allowbreak\symb)}$} & mean$\pm$std & 0.0 $\pm$ 0.0 & 0.0 $\pm$ 0.0 & \textbf{500.0} $\pm$ 0.0 & 50.4 $\pm$ 2.9 & 51.2 $\pm$ 3.8 & 51.1 $\pm$ 4.2 & 68.1 $\pm$ 18.8 \\
 &  & max & 0.0 & 0.0 & \textbf{500.0} & 56.0 & 58.0 & 58.0 & 106.0 \\
 & \multirow{2}{*}{$\kleene{\{\syma,\symb,\symd\}}\allowbreak\syma\allowbreak\kleene{\{\symc,\symd\}}$} & mean$\pm$std & 0.0 $\pm$ 0.0 & 0.0 $\pm$ 0.0 & 0.0 $\pm$ 0.0 & 47.6 $\pm$ 4.1 & 53.1 $\pm$ 6.2 & 63.9 $\pm$ 14.2 & 45.9 $\pm$ 2.5 \\
 &  & max & 0.0 & 0.0 & 0.0 & 58.0 & 67.0 & 103.0 & 51.0 \\
\midrule
\multirow{4}{*}{$\yptl$} & \multirow{2}{*}{$\kleene{\alphabet}\allowbreak\syma\allowbreak\symb\allowbreak\kleene{\alphabet}$} & mean$\pm$std & 0.0 $\pm$ 0.0 & 0.0 $\pm$ 0.0 & 0.0 $\pm$ 0.0 & 189.6 $\pm$ 62.2 & 190.9 $\pm$ 70.1 & 258.7 $\pm$ 140.0 & 45.7 $\pm$ 1.9 \\
 &  & max & 0.0 & 0.0 & 0.0 & 355.0 & 294.0 & \textbf{500.0} & 50.0 \\
 & \multirow{2}{*}{$\kleene{(\syma\symb)}$} & mean$\pm$std & 0.0 $\pm$ 0.0 & 0.0 $\pm$ 0.0 & 0.0 $\pm$ 0.0 & 186.4 $\pm$ 67.1 & 183.5 $\pm$ 96.2 & 145.9 $\pm$ 79.6 & 59.9 $\pm$ 17.0 \\
 &  & max & 0.0 & 0.0 & 0.0 & 314.0 & \textbf{500.0} & 344.0 & 110.0 \\
\midrule
\multirow{4}{*}{$\ptl$} & \multirow{2}{*}{$\kleene{\alphabet}\allowbreak\syma\allowbreak\kleene{\alphabet}\allowbreak\symb\allowbreak\kleene{\alphabet}$} & mean$\pm$std & 0.0 $\pm$ 0.0 & 0.0 $\pm$ 0.0 & 0.0 $\pm$ 0.0 & 260.0 $\pm$ 71.0 & 356.5 $\pm$ 101.9 & 285.7 $\pm$ 119.9 & 279.3 $\pm$ 54.1 \\
 &  & max & 0.0 & 0.0 & 0.0 & 410.0 & \textbf{500.0} & \textbf{500.0} & 406.0 \\
 & \multirow{2}{*}{$\syma\kleene{\alphabet}$} & mean$\pm$std & 0.0 $\pm$ 0.0 & 0.0 $\pm$ 0.0 & 0.0 $\pm$ 0.0 & 134.8 $\pm$ 53.9 & 129.4 $\pm$ 50.0 & 162.1 $\pm$ 100.2 & 156.1 $\pm$ 46.9 \\
 &  & max & 0.0 & 0.0 & 0.0 & 282.0 & 234.0 & \textbf{500.0} & 260.0 \\
\midrule
\multirow{4}{*}{$\ytl$} & \multirow{2}{*}{$\kleene{\alphabet}\syma\symb$} & mean$\pm$std & \textbf{500.0} $\pm$ 0.0 & \textbf{500.0} $\pm$ 0.0 & \textbf{500.0} $\pm$ 0.0 & 490.9 $\pm$ 33.9 & \textbf{500.0} $\pm$ 0.0 & \textbf{500.0} $\pm$ 0.0 & 45.1 $\pm$ 0.8 \\
 &  & max & \textbf{500.0} & \textbf{500.0} & \textbf{500.0} & \textbf{500.0} & \textbf{500.0} & \textbf{500.0} & 46.0 \\
 & \multirow{2}{*}{$\kleene{\alphabet}\syma$} & mean$\pm$std & \textbf{500.0} $\pm$ 0.0 & \textbf{500.0} $\pm$ 0.0 & \textbf{500.0} $\pm$ 0.0 & \textbf{500.0} $\pm$ 0.0 & \textbf{500.0} $\pm$ 0.0 & \textbf{500.0} $\pm$ 0.0 & 45.6 $\pm$ 2.2 \\
 &  & max & \textbf{500.0} & \textbf{500.0} & \textbf{500.0} & \textbf{500.0} & \textbf{500.0} & \textbf{500.0} & 51.0 \\
\bottomrule
\end{tabular}
\caption{Longest perfect length on the formal-language tasks with $\rotary$. }
\label{tab:formal_languages_len_rope}
\end{table*}

\begin{table*}[t]
\centering
\scriptsize
\setlength{\tabcolsep}{3pt}
\renewcommand{\arraystretch}{1.08}
\begin{tabular}{lllccccccc}
\toprule
\multicolumn{3}{c}{} & \multicolumn{3}{c}{Local} & \multicolumn{3}{c}{Hybrid} & \multicolumn{1}{c}{Global} \\
\cmidrule(lr){4-6}\cmidrule(lr){7-9}\cmidrule(lr){10-10}
\multicolumn{3}{c}{} & $k=1$ & $k=2$ & $k=4$ & $k=1$ & $k=2$ & $k=4$ & $-$ \\
\midrule
\multirow{4}{*}{$\ltl$} & \multirow{2}{*}{$\kleene{(\syma\allowbreak\kleene{(\syma\symb)}\allowbreak\symb)}$} & mean$\pm$std & 0.0 $\pm$ 0.0 & 0.0 $\pm$ 0.0 & 42.0 $\pm$ 0.0 & 64.9 $\pm$ 23.5 & 63.9 $\pm$ 21.4 & 56.0 $\pm$ 20.1 & 42.0 $\pm$ 0.0 \\
 &  & max & 0.0 & 0.0 & 42.0 & 122.0 & 102.0 & 90.0 & 42.0 \\
 & \multirow{2}{*}{$\kleene{\{\syma,\symb,\symd\}}\allowbreak\syma\allowbreak\kleene{\{\symc,\symd\}}$} & mean$\pm$std & 0.0 $\pm$ 0.0 & 0.0 $\pm$ 0.0 & 0.0 $\pm$ 0.0 & 56.0 $\pm$ 7.7 & 65.9 $\pm$ 34.3 & 51.5 $\pm$ 14.4 & 41.1 $\pm$ 0.3 \\
 &  & max & 0.0 & 0.0 & 0.0 & 71.0 & 150.0 & 77.0 & 42.0 \\
\midrule
\multirow{4}{*}{$\yptl$} & \multirow{2}{*}{$\kleene{\alphabet}\allowbreak\syma\allowbreak\symb\allowbreak\kleene{\alphabet}$} & mean$\pm$std & 0.0 $\pm$ 0.0 & 0.0 $\pm$ 0.0 & 0.0 $\pm$ 0.0 & 50.9 $\pm$ 12.4 & 95.8 $\pm$ 86.2 & 41.6 $\pm$ 0.5 & 41.0 $\pm$ 0.0 \\
 &  & max & 0.0 & 0.0 & 0.0 & 84.0 & 330.0 & 42.0 & 41.0 \\
 & \multirow{2}{*}{$\kleene{(\syma\symb)}$} & mean$\pm$std & 0.0 $\pm$ 0.0 & 0.0 $\pm$ 0.0 & 0.0 $\pm$ 0.0 & 101.5 $\pm$ 41.8 & 121.5 $\pm$ 111.7 & 65.3 $\pm$ 38.4 & 42.4 $\pm$ 0.8 \\
 &  & max & 0.0 & 0.0 & 0.0 & 176.0 & \textbf{500.0} & 176.0 & 44.0 \\
\midrule
\multirow{4}{*}{$\ptl$} & \multirow{2}{*}{$\kleene{\alphabet}\allowbreak\syma\allowbreak\kleene{\alphabet}\allowbreak\symb\allowbreak\kleene{\alphabet}$} & mean$\pm$std & 0.0 $\pm$ 0.0 & 0.0 $\pm$ 0.0 & 0.0 $\pm$ 0.0 & 154.3 $\pm$ 143.2 & 183.0 $\pm$ 175.8 & 151.4 $\pm$ 154.6 & 218.1 $\pm$ 191.9 \\
 &  & max & 0.0 & 0.0 & 0.0 & \textbf{500.0} & \textbf{500.0} & \textbf{500.0} & \textbf{500.0} \\
 & \multirow{2}{*}{$\syma\kleene{\alphabet}$} & mean$\pm$std & 0.0 $\pm$ 0.0 & 0.0 $\pm$ 0.0 & 0.0 $\pm$ 0.0 & 420.9 $\pm$ 99.8 & 420.3 $\pm$ 96.0 & 336.5 $\pm$ 132.3 & 446.9 $\pm$ 107.3 \\
 &  & max & 0.0 & 0.0 & 0.0 & \textbf{500.0} & \textbf{500.0} & \textbf{500.0} & \textbf{500.0} \\
\midrule
\multirow{4}{*}{$\ytl$} & \multirow{2}{*}{$\kleene{\alphabet}\syma\symb$} & mean$\pm$std & 323.2 $\pm$ 174.5 & 41.0 $\pm$ 0.0 & 41.2 $\pm$ 0.4 & 281.0 $\pm$ 179.9 & 41.0 $\pm$ 0.0 & 41.2 $\pm$ 0.4 & 41.0 $\pm$ 0.0 \\
 &  & max & \textbf{500.0} & 41.0 & 42.0 & \textbf{500.0} & 41.0 & 42.0 & 41.0 \\
 & \multirow{2}{*}{$\kleene{\alphabet}\syma$} & mean$\pm$std & 198.3 $\pm$ 154.9 & 41.3 $\pm$ 0.5 & 41.0 $\pm$ 0.0 & 218.9 $\pm$ 148.8 & 41.1 $\pm$ 0.3 & 41.0 $\pm$ 0.0 & 41.0 $\pm$ 0.0 \\
 &  & max & \textbf{500.0} & 42.0 & 41.0 & \textbf{500.0} & 42.0 & 41.0 & 41.0 \\
\bottomrule
\end{tabular}
\caption{Longest perfect length on the formal-language tasks with $\sincos$. }
\label{tab:formal_languages_len_sipe}
\end{table*}

\end{document}